\pdfoutput=1
\documentclass{article}
\usepackage[final, nonatbib]{neurips_2021}
\usepackage[utf8x]{inputenc} 
\usepackage[T1]{fontenc}    
\usepackage{hyperref}       
\usepackage{url}            
\usepackage{booktabs}       
\usepackage{amsfonts}       
\usepackage{nicefrac}       
\usepackage{microtype}      
\usepackage{fixltx2e}
\usepackage{subcaption}
\usepackage[cmex10]{amsmath}
\usepackage{amssymb}
\usepackage{amsfonts}
\usepackage{amsmath}
\usepackage{supertabular}
\usepackage{times}
\usepackage{amsthm}
\usepackage{floatrow}
\usepackage{wrapfig}
\usepackage{tikz}
\usepackage{float}
\usepackage{mathtools}
\usepackage{textcomp}
\usepackage{multirow}
\usepackage{algpseudocode}
\usepackage[ruled,vlined]{algorithm2e}
\usepackage{adjustbox}
\usepackage[justification=centering, font=small]{caption} 
\usepackage{epstopdf}
\usepackage{tabularx}
\usepackage{nccmath}
\usepackage{mathtools}
\usepackage{easyReview}
\usepackage{tikz}

\theoremstyle{plain}
\newtheorem{thm}{Theorem}

\newtheorem{lemma}{Lemma}

\newtheorem{prop}{Proposition}
\newtheorem{corr}{Corollary}

\newcommand{\vect}[1]{\boldsymbol{#1}}

\newcommand{\cT}{\vect{\mathcal{Q}}} 
\newcommand{\cX}{\vect{\mathcal{X}} } %
\newcommand{\cD}{\vect{\mathcal{D}} }
\newcommand{\mR}{\vect{\mathbb{R}} }
\newcommand{\mN}{\vect{\mathbb{N}} }

\graphicspath{{./Figures/}}

\title{Formalizing the Generalization-Forgetting Trade-Off in Continual Learning}
\author{R.~Krishnan$^{1}$ and Prasanna~Balaprakash$^{1,2}$  \\
    $\text{ }^{1}$Mathematics and Computer Science Division\\
    $\text{ }^{2}$Leadership Computing Facility\\ 
    Argonne National Laboratory\\
  \textit{kraghavan,pbalapra@anl.gov}} 
\begin{document}
\maketitle
\begin{abstract}
We formulate the continual learning problem via dynamic programming and model the trade-off between catastrophic forgetting and generalization as a two-player sequential game. In this approach, player 1 maximizes the cost due to lack of generalization whereas player 2 minimizes the cost due to increased catastrophic forgetting. We show theoretically and experimentally that a balance point between the two players exists for each task and that this point is stable~(once the balance is achieved, the two players stay at the balance point). Next, we introduce balanced continual learning (BCL), which is designed to attain balance between generalization and forgetting, and we empirically demonstrate that BCL is comparable to or better than the state of the art.
\end{abstract}

\section{Introduction}






In continual learning~(CL), we incrementally adapt a model to learn tasks~(defined according to the problem at hand) observed sequentially. CL has two main objectives: maintain long-term memory~(remember previous tasks) and navigate new experiences continually~(quickly adapt to new tasks). An important characterization of these objectives is provided by the stability-plasticity dilemma~\cite{carpenter1987massively}, where the primary challenge is to balance network stability~(preserve past knowledge; minimize catastrophic forgetting) and plasticity~(rapidly learn from new experiences; generalize quickly). This balance provides a natural objective for CL: \textit{balance forgetting and generalization.} 

Traditional CL methods either minimize catastrophic forgetting or improve quick generalization but do not model both. For example, common solutions to the catastrophic forgetting issue include (1) representation-driven approaches~\cite{yoon2017lifelong, ke2020continual}, (2) regularization approaches~\cite{kirkpatrick2017overcoming, aljundi2018memory, mirzadeh2020understanding, farajtabar2019orthogonal,  yin2020optimization, yin2020optimization, jung2020continual, pan2020continual, chaudhry2020continual, titsias2019functional}, and (3) memory/experience replay~\cite{lin1992self, lopez2017gradient, chaudhry2019continual, chaudhry2019tiny, fini2020online}. Solutions to the generalization problem include representation-learning approaches~(matching nets~\cite{vinyals2016matching}, prototypical networks~\cite{snell2017prototypical}, and metalearning approaches~\cite{finn2017model, finn2019online, Caccia2020OnlineFA, yao2020don}). More recently, several approaches \cite{nagabandi2019deep,farajtabar2019orthogonal,DBLP:journals/corr/abs-1806-06928, joseph2020metaconsolidation, yin2020optimization, ebrahimi2020adversarial} have been introduced that combine methods designed for quick generalization with frameworks designed to minimize forgetting.

The aforementioned CL approaches naively minimize a loss function~(combination of forgetting and generalization loss) but do not explicitly account for the trade-off in their optimization setup. The first work to formalize this trade-off was presented in meta-experience replay~(MER)~\cite{riemer2018learning}, where the forgetting-generalization trade-off was posed as a gradient alignment problem. Although MER provides a promising methodology for CL, the balance between forgetting and generalization is enforced with several hyperparameters. 
Therefore, two key challenges arise:  (1)~lack of theoretical tools that study the existence~(\textit{under what conditions does a balance point between generalization and forgetting exists?}) and stability~(\textit{can this balance be realistically achieved?}) of a balance point and (2)~lack of a systematic approach to achieve the balance point. We address these challenges in this paper.

We describe a framework where we first formulate CL as a sequential decision-making problem and seek to minimize a cost function summed over the complete lifetime of the model. At any time $k,$ given that the future tasks are not available, the calculation of the cost function becomes intractable. To circumvent this issue, we use  Bellman's principle of optimality~\cite{bellman2015adaptive} and recast the CL problem to model the catastrophic forgetting cost on the previous tasks and generalization cost on the new task. We show that equivalent performance on an infinite number of tasks is not practical~(Lemma~1 and Corollary~1) and that tasks must be prioritized.
\begin{figure}[!h]
\begin{subfigure}{\textwidth}
  \centering
	\includegraphics[width = \columnwidth]{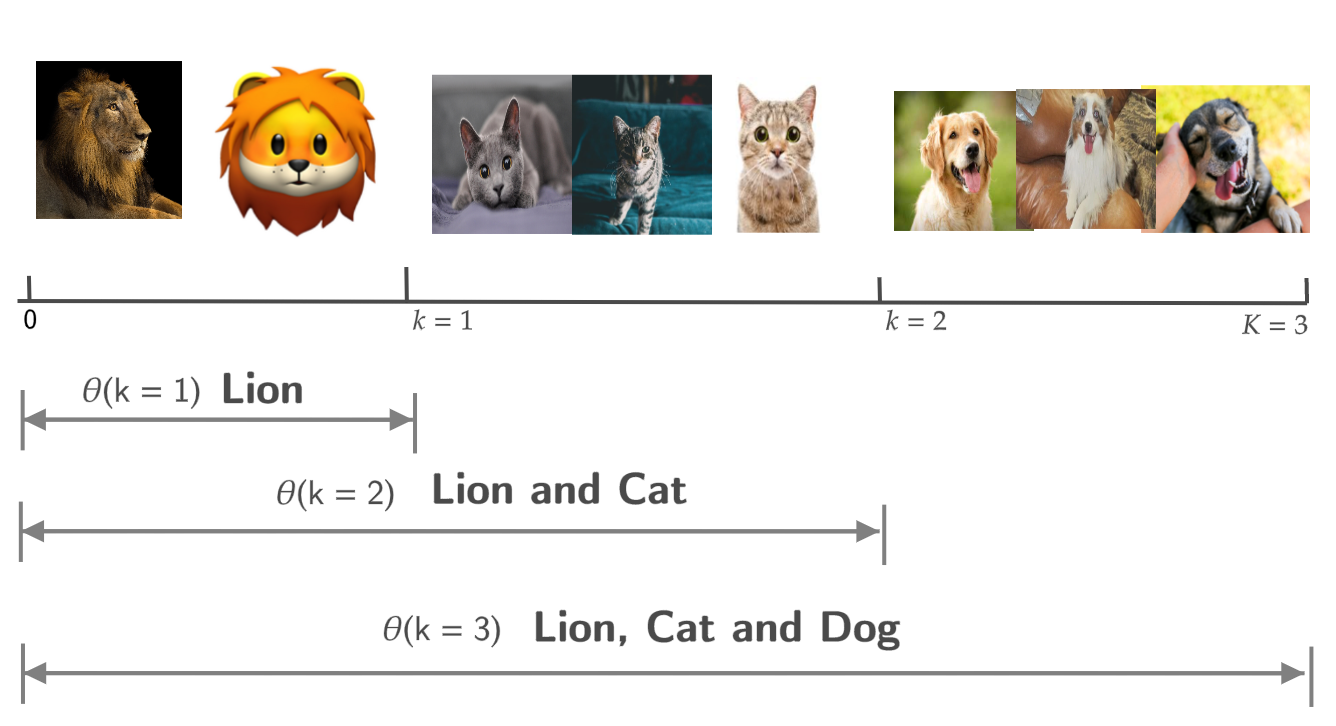}
  \label{fig:ills}
\end{subfigure}\\
\begin{subfigure}{\textwidth}
 \centering
	\includegraphics[width = \columnwidth ]{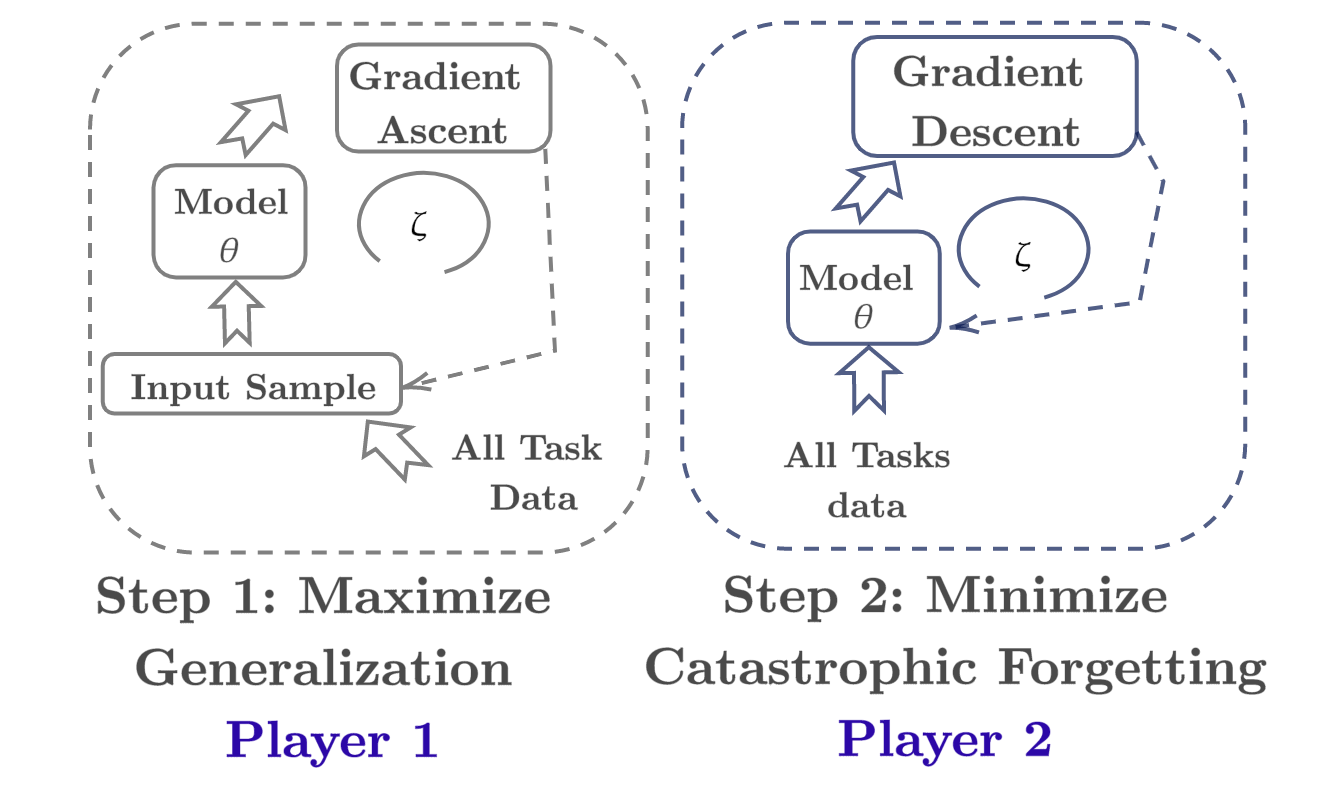}
	 \label{fig:met}
\end{subfigure}
\caption{(left) Exemplary CL problem: the lifetime of the model can be split into three intervals. At $k =1$ we seek to recognize lions; at $k=2$ we seek to recognize both lions and cats; and  at $k=3$ we seek to recognize cats, lions, and dogs. (right) Illustration of the proposed method: our methodology  comprises  an interplay between two players. The first player maximizes generalization by simulating maximum discrepancy between two tasks. The second player minimizes forgetting by adapting to maximum discrepancy.}
\label{fig:ills}
\end{figure}
To achieve a balance between forgetting and generalization, we pose the trade-off
as a saddle point problem where we designate one player for maximizing the generalization cost~(player 1) and another for minimizing the forgetting cost~(player 2). We prove mathematically that there exists at least one saddle point between generalization and forgetting for each new task~(Theorem~1). Furthermore, we show that this saddle point can be attained asymptotically~(Theorem~2) when player strategies are chosen as gradient ascent-descent. We then introduce balanced continual learning~(BCL), a new algorithm to achieve this saddle point. In our algorithm~(see Fig.~\ref{fig:ills} for a description of BCL), the generalization cost is computed by training and evaluating the model on given new task data. The catastrophic forgetting cost is computed by evaluating the model on the task memory (previous tasks). We first maximize the generalization cost and then minimize the catastrophic forgetting cost to achieve the balance. We compare our approach with other methods such as elastic weight consolidation~(EWC)~\cite{kirkpatrick2017overcoming}, online EWC~\cite{schwarz2018progress}, and MER~\cite{riemer2018learning} on continual learning benchmark data sets~\cite{Hsu18_EvalCL} to show that BCL is  better than or comparable to the state-of-the-art methods. Moreover, we also show in simulation that our theoretical framework is appropriate for understanding the continual learning problem. The contributions of this paper are (1) a theoretical framework to study the CL problem, (2) BCL, a method to attain balance between forgetting and generalization, and (3) advancement of the state of the art in CL.

\section{Problem Formulation}

We use $\mR$ to denote the set of real numbers and $\mN$ to denote the set of natural numbers. We  use $\|.\|$ to denote the Euclidean norm for vectors and the Frobenius norm for matrices, while using bold symbols to illustrate matrices and vectors. We define an interval $[0, K), K \in \mN$ and  let $p(\cT)$ be the distribution over all the tasks observed in this interval. For any $k \in [0, K) ,$ we define a parametric model $g(.)$ with ${\vect{y}}_{k} = g(\vect{x}_{k}; {\vect{\theta}}_{k})$, where ${\vect{\theta}}_{k}$ is a vector comprising all parameters of the model with $\vect{x}_{k} \in \cX_{k}$.  Let $n$ be the number of samples  and  $m$ be the number of dimensions. Suppose a task at $k | k \in [0, K)$ is observed and denoted as  $\cT_{k} : \cT_{k} \sim p(\cT)$, where $\cT_{k} =\{\cX_{k}, \ell_{k}\}$ is a tuple with $\cX_{k} \in \mR^{n \times m}$  being the input data and $\ell_{k}$ quantifies the loss incurred by $\cX_{k}$ using the model $g$ for the task at $k$.   We denote a sequence of ${\vect{\theta}}_{k}$ as $\vect{u}_{k:K} = \{{\vect{\theta}}_{\tau} \in \Omega_{\theta}, k \leq \tau \leq K \},$ with  $\Omega_{\theta}$ being the compact~(feasible) set for the parameters. We denote the optimal value with a superscript ${(*)};$ for instance, we use ${\vect{\theta}}_k^{(*)}$ to denote the optimal value of ${\vect{\theta}}_{k}$ at task $k.$ In this paper we  use balance point, equilibrium point, and saddle  point to refer to the point of balance between generalization and forgetting. We  interchange between these terms whenever convenient for the discussion. We will use $\nabla_{(j)} i$ to denote the gradient of $i$ with respect to $j$ and $\Delta i$ to denote the first difference in discrete time.

An exemplary CL problem is described in Fig.~\ref{fig:ills} where we  address a total of $K=3$ tasks. To particularize the idea in Fig. \ref{fig:ills}, we define the cost~(combination of catastrophic cost and generalization cost) at any instant $k$ as 

\[J_{k}({\vect{\theta} }_{k}) = \gamma_{k} \ell_{k} + \sum_{\tau = 0}^{k-1} \gamma_{\tau} \ell_{\tau},\] 

where $\ell_{\tau}$ is computed on task $\cT_{\tau}$ with $\gamma_{\tau}$ describing the contribution of $\cT_{\tau}$ to this sum. 

To solve the problem at $k$, we seek ${\vect{\theta} }_{k}$ to minimize $J_{k}({\vect{\theta} }_{k})$. Similarly, to solve the problem in the complete interval $[0,K]$, we seek a ${\vect{\theta} }_{k}$ to minimize $J_{k}({\vect{\theta} }_{k})$ for each $k \in [0,K].$ In other words we seek to obtain ${\vect{\theta} }_{k}$ for each task such that the cost $J_{k}({\vect{\theta} }_{k})$ is minimized. Therefore, the optimization problem for the overall CL problem~(overarching goal of CL) is provided as the minimization of the cumulative cost \[ V_{k}(\vect{u}_{k:K}) = \sum_{\tau=k}^{K} \beta_{\tau} J_{\tau}( {\vect{\theta}}_{\tau}) \] such that $V_k^{(*)},$  is given as  \begin{equation} V_k^{(*)} = min_{\vect{u}_{k:K} }  V_{k}(\vect{u}_{k:K}), \label{op_1} \end{equation} with $0 \leq \beta_{\tau} \leq 1$ being the contribution of $J_{\tau}$ and $\vect{u}_{k:K}$ being a weight sequence of length $K-k.$ 

Within this formulation,  two parameters  determine the contributions of tasks: $\gamma_{\tau}$, the contribution of each task in the past, and $\beta_{\tau}$, the contribution of tasks in the future. To successfully solve the optimization problem, $V_{k}(\vect{u}_{k:K})$ must be bounded and differentiable, typically ensured by the choice of $\gamma_{\tau}, \beta_{\tau}.$ Lemma~1~(full statement and proof in Appendix A) states that \textit{equivalent performance cannot be guaranteed for an infinite number of tasks}.  Furthermore, Corollary~1~(full statement and proof in Appendix A) demonstrates that \textit{if the task contributions are prioritized, the differentiability and boundedness of $J_\tau({\vect{\theta}}_{\tau})$ can be ensured}. A similar result was proved in \cite{knoblauch2020optimal}, where a CL problem with infinite memory was shown to be NP-hard from a set theoretic perspective. These results~(both ours and in \cite{knoblauch2020optimal}) demonstrate that a CL methodology cannot provide perfect performance on a large number of tasks and that tasks must be prioritized.

Despite these invaluable insights, the data corresponding to future tasks~(interval $[k,K]$) is not available, and therefore $V_{k}(\vect{u}_{k:K})$ cannot be evaluated. The optimization problem in Eq.~\eqref{op_1} naively minimizes the cost~(due to both previous tasks and new tasks) and does not provide any explicit modeling of the trade-off between forgetting and generalization. Furthermore, $\vect{u}_{k:K},$ the solution to Eq. \eqref{op_1} is a sequence of parameters, and it is not feasible to maintain $\vect{u}_{k:K}$ for a large number of tasks. Because of these three issues, the problem is theoretically intractable in its current form. 

We will first recast the problem using tools from dynamic programming~\cite{lewis2012optimal}, specifically Bellman's principle of optimality, and derive a difference equation that summarizes the complete dynamics for the CL problem. Then, we will formulate a two-player differential game where we seek a saddle point solution to balance generalization and forgetting. 

\section{Dynamics of Continual Learning}
Let \[ V_k^{(*)} = min_{\vect{u}_{k:K} }  \sum_{\tau=k}^{K} \beta_{\tau} J_{\tau}( {\vect{\theta}}_{\tau});\] the dynamics of CL~(the behavior of optimal cost with respect to $k$) is provided as
\begin{equation}
	\begin{aligned}
    \Delta V^{(*)}_{k}  =	- min_{{\vect{\theta}}_{k} \in \Omega_{\theta}}   \big[ \beta_k J_{k}( {\vect{\theta}}_{k})  +  \big( \langle \nabla_{{\vect{\theta}}_{k}} V_{k}^{(*)} , \Delta {\vect{\theta}}_{k} \rangle  +\langle \nabla_{\vect{x}_{k}} V_{k}^{(*)}, \Delta \vect{x}_{k} \rangle  \big)\big].\\ 
	\end{aligned}
	\label{eq_M_DES}
\end{equation}
The derivation is presented in Appendix A~(refer to Proposition~1). Note that $V^{(*)}_{k}$ is the minima for the overarching CL problem in Eq. \eqref{eq_M_DES}and $\Delta V^{(*)}_{k}$ represents the change in $V^{(*)}_{k}$ upon introduction of a new task~(we hitherto refer to this as perturbations). Zero perturbations $(\Delta V^{(*)}_{k}=0)$ implies that the introduction of a new task does not impact our current solution; that is, the optimal solution on all previous tasks is optimal on the new task as well. Therefore, the smaller the perturbations, the better the performance of a model on all tasks, thus providing our main objective: minimize the perturbations~($\Delta V^{(*)}_{k} $). In Eq. \ref{eq_M_DES}, $\Delta V^{(*)}_{k} $ is quantified by three terms: the cost contribution from all the previous tasks and the new task~$J_{k}( {\vect{\theta}}_{k});$  the change in the optimal cost due to the change in the parameters $\langle \nabla_{{\vect{\theta}}_{k}} V_{k}^{(*)}, \Delta {\vect{\theta}}_{k} \rangle$; and the change in the optimal cost due to the change in the input (introduction of new task)~$\langle \nabla_{\vect{x}_{k}} V_{k}^{(*)}, \Delta \vect{x}_{k}\rangle$. 

The first issue with the cumulative CL problem~(Eq. \eqref{op_1}) can be attributed to the need for information from the future. In Eq. \eqref{eq_M_DES}, all information from the future is approximated by using the data from the new and the previous tasks.  Therefore, the solution of the CL problem can directly be obtained by solving Eq. \eqref{eq_M_DES} using all the available data. Thus, \[ min_{{\vect{\theta}}_{k} \in \Omega}   \big[ H(\Delta \vect{x}_{k}, \vect{\theta}_{k})  \big] \quad \text{yields } \Delta V^{(*)}_{k}  \approx 0 \] for $\beta > 0,$ with \[ H(\Delta \vect{x}_{k}, \vect{\theta}_{k}) =  \beta_k J_{k}( {\vect{\theta}}_{k}) +   \langle \nabla_{{\vect{\theta}}_{k}} V_{k}^{(*)} , \Delta {\vect{\theta}}_{k} \rangle + \langle \nabla_{\vect{x}_{k}} V_{k}^{(*)}, \Delta \vect{x}_{k} \rangle.\] Essentially, minimizing $H(\Delta \vect{x}_{k}, \vect{\theta}_{k})$  would minimize the perturbations introduced by any new task $k$. 

In Eq. \eqref{eq_M_DES}, the first and the third term quantify generalization and the second term quantifies forgetting. A model exhibits generalization when it successfully adapts to a new task~(minimizes the first and the third term in Eq. \eqref{eq_M_DES}). The degree of generalization depends on the discrepancy between the previous tasks and the new task~(numerical value of the third term in Eq. \eqref{eq_M_DES}) and the worst-case discrepancy prompts maximum generalization. Quantification of generalization is provided by $\Delta \vect{x}_{k}$ that summarizes the discrepancy between subsequent tasks. However, $\Delta \vect{x}_{k} = \vect{x}_{k+1} - \vect{x}_{k}$, and $\vect{x}_{k+1}$ is unknown at $k.$ Therefore, we simulate  worst-case discrepancy by iteratively updating $\Delta \vect{x}_{k}$ through gradient ascent in order to maximize $H(\Delta \vect{x}_{k}, \vect{\theta}_{k})$; thus maximizing generalization. However, large discrepancy increases forgetting, and worst-case discrepancy yields maximum forgetting. Therefore, once maximum generalization is simulated, minimizing forgetting~(update $\vect{\theta}_{k}$ by gradient descent) under maximum generalization provides the balance.

To formalize our idea, let us indicate the iteration index at $k$ by $i$ and write $\Delta \vect{x}_{k}$ as $\Delta \vect{x}^{(i)}_{k}$ and ${\vect{\theta}}_{k}$ as ${\vect{\theta}}^{(i)}_{k}$ with $H(\Delta \vect{x}_{k}, \vect{\theta}_{k})$ as $H(\Delta \vect{x}^{(i)}_{k}, \vect{\theta}^{(i)}_{k})$~(for simplicity of notation, we will denote $H(\Delta \vect{x}^{(i)}_{k}, \vect{\theta}^{(i)}_{k})$ as $H$ whenever convenient). Next, we write
\begin{equation}
	\begin{aligned}
	\centering
     &\underset{{\vect{\theta}}^{(i)}_{k} \in \Omega_{\theta} }{min}   \bigg[ H(\Delta \vect{x}^{(i)}_{k}, \vect{\theta}^{(i)}_{k}) \bigg]  = \underset{{\vect{\theta}}^{(i)}_{k} \in \Omega_{\theta}}{min}   \big[ \beta_k  J_{k}( {\vect{\theta}}^{(i)}_{k}) +   \langle \nabla_{{\vect{\theta}}^{(i)}_{k}} V_{k}^{(*)} , \Delta {\vect{\theta}}^{(i)}_{k}\rangle + \langle \nabla_{\vect{x}^{(i)}_{k}} V_{k}^{(*)}, \Delta \vect{x}^{(i)}_{k} \rangle \big]& \\ &\leq \underset{{\vect{\theta}}^{(i)}_{k} \in \Omega_{\theta}}{min}   \big[\beta_k  J_{k}( {\vect{\theta}}^{(i)}_{k}) +   \langle \nabla_{{\vect{\theta}}^{(i)}_{k}} V_{k}^{(*)} , \Delta {\vect{\theta}}^{(i)}_{k}\rangle + \underset{\Delta \vect{x}^{(i)}_{k} \sim p(\cT)}{max}  \langle \nabla_{\vect{x}^{(i)}_{k}} V_{k}^{(*)}, \Delta \vect{x}^{(i)}_{k} \rangle \big] &\\  & \leq \underset{{\vect{\theta}}^{(i)}_{k} \in \Omega_{\theta} }{min} \quad \underset{\Delta \vect{x}^{(i)}_{k} \sim p(\cT)}{max}  \big[ H(\Delta \vect{x}^{(i)}_{k}, \vect{\theta}^{(i)}_{k})\big]. & \label{eq_op_2}
	\end{aligned}
\end{equation}
 In Eq. \eqref{eq_op_2}, we seek the solution pair $(\Delta \vect{x}_k^{(*)}, \vect{\theta}_k^{(*)} ) \in (\Omega_{\theta}, \Omega_{\Delta \vect{x}_k^{(*)}}),$ where $\Delta \vect{x}_k^{(*)}$ maximizes $H$~(maximizing player, player 1) while $\vect{\theta}_k^{(*)}$ minimizes $H$~(minimizing player, player 2) where $(\Omega_{\theta}, \Omega_{\Delta \vect{x}_k^{(*)}})$ are the feasible sets for $\Delta \vect{x}_k^{(i)}$ and $\vect{\theta}_k^{(i)}$ respectively. The solution is attained, and  $(\Delta \vect{x}_k^{(*)}, \vect{\theta}_k^{(*)})$ is said to be the equilibrium point when it satisfies the following condition:
\begin{equation}
	\begin{aligned}
        H(\Delta \vect{x}_k^{(*)}, \vect{\theta}^{(i)}_{k})
        \geq   H(\Delta \vect{x}_k^{(*)}, \vect{\theta}_k^{(*)} ) \geq H(\Delta \vect{x}^{(i)}_{k}, \vect{\theta}_k^{(*)} ).
	\end{aligned}
	\label{eq_condition}
\end{equation}
\subsection{Theoretical Analysis\label{theory}}
 \begin{wrapfigure}[17]{l}{0.51\textwidth}
    \includegraphics[width = \columnwidth]{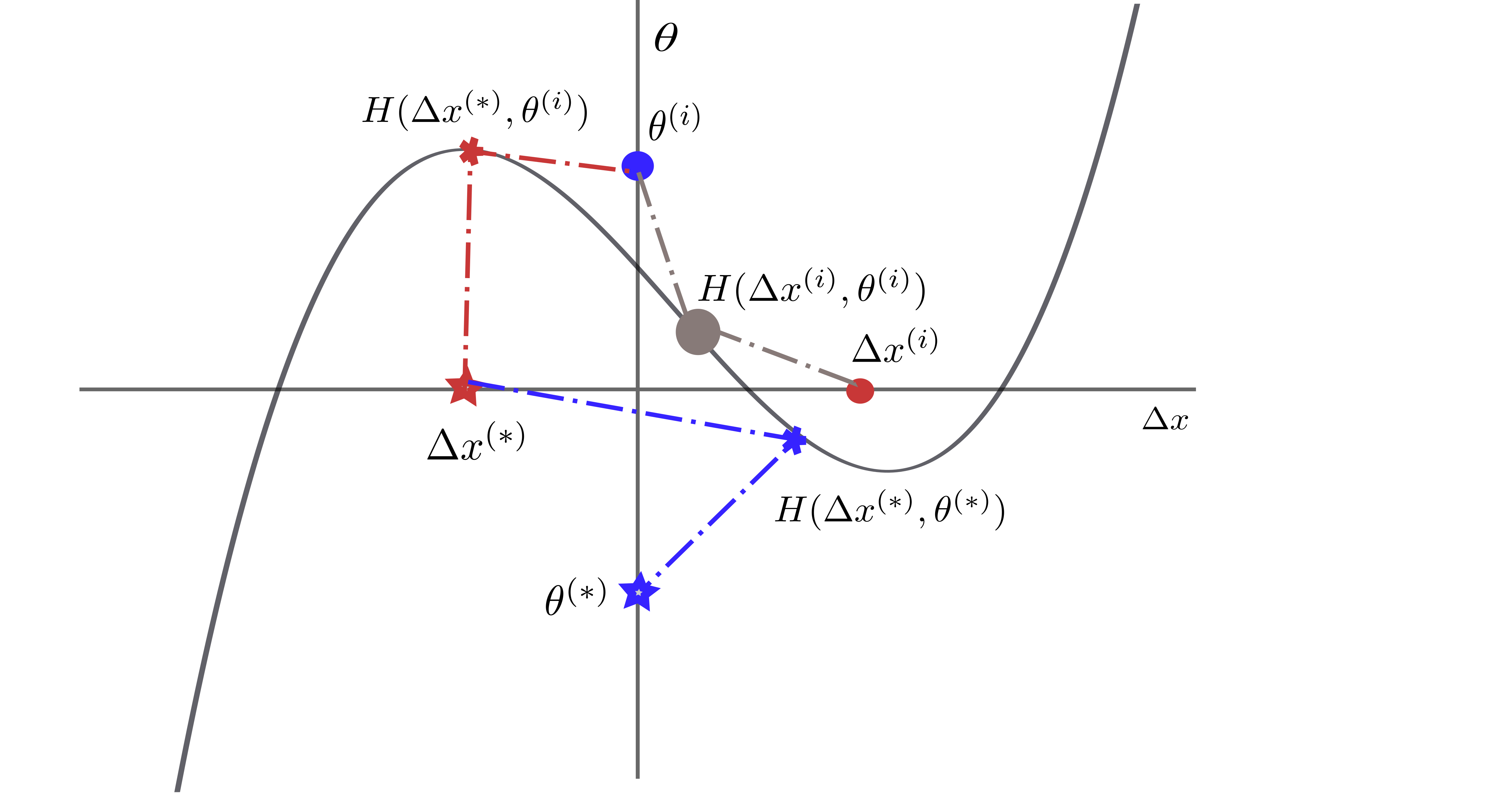}
    \caption{Illustration of  the proofs. $\Delta \vect{x}$~(player~1) is the horizontal axis, and the vertical axis indicates $\vect{\theta}$~(player~2) where the curve indicates H. If we start from the red circle for player~1~(player~2 is fixed at the  blue circle), H is increasing~(goes from a grey circle to a red asterisk) with player 1 reaching the red asterisk. Next,  start from the blue circle~($\vect{\theta}$ is at the red asterisk), the cost decreases. }
    \label{fig:proof}
\end{wrapfigure} With our formulation, two key questions arise: Does our problem setup have an equilibrium point satisfying Eq. \eqref{eq_condition}? and how can one attain this equilibrium point? We answer these questions with Theorems~1 and 2, respectively. Full statements and proofs are provided in Appendix~A. 
 
 To illustrate the theory, we refer to Fig.~\ref{fig:proof}, where the initial values for the two players are characterized by the pair  $\{\vect{\theta}^{(i)}_{k} \text{(blue circle)}, \Delta \vect{x}^{(i)}_{k} \text{(red circle)} \}$ and the cost value at $\{\vect{\theta}^{(i)}_{k}, \Delta \vect{x}^{(i)}_{k} \}$  is indicated by $H( \Delta \vect{x}^{(i)}_{k}, \vect{\theta}^{(i)}_{k})$~(the grey circle on the cost curve~(the dark blue curve)). Our proofing strategy is as follows. First, we fix $\vect{\theta}_k^{(.)} \in \Omega_{\theta}$ and construct a neighborhood $\mathcal{M}_k= \{\Omega_{x}, \vect{\theta}^{(.)}_{k}\}.$ Within this neighborhood we prove in Lemmas~2 and 4 that if we search for $\Delta \vect{x}^{(i)}_{k}$ through gradient ascent, we can converge to a local maximizer, and $H$ is maximizing with respect to $\Delta \vect{x}^{(i)}_{k}.$ 
  Second, we let $\Delta \vect{x}^{(.)}_{k} \in \Omega_x$ be fixed, and we search for $ \vect{\theta}^{(i)}_{k}$ through gradient descent. Under this condition, we demonstrate two ideas in Lemmas 3 and 5: (1) we show that  $H$ is minimizing in the neighborhood $\mathcal{N}_k: \mathcal{N}_k= \{\Omega_{\theta} , \Delta \vect{x}^{(.)}_{k} \}$; and (2) we converge to the local minimizer in the neighborhood $\mathcal{N}_k.$  Third, in the union of the two neighborhoods $\mathcal{M}_k \cup \mathcal{N}_k,$ (proven to be nonempty according to Lemma~6), we show that there exists at least one local equilibrium point~(Theorem~1); that is, there is at least one balance point.

\begin{thm}[Existence of an Equilibrium Point]
For any $k \in [0,K]$, let $\vect{\theta}^{(*)}_{k} \in \Omega_\theta,$  be the minimizer of $H$ according to Lemma 5 and define $\mathcal{M}^{(*)}_k = \{\Omega_{x}, \vect{\theta}^{(*)}_{k}\}.$ Similarly, let $\Delta \vect{x}_k^{(*)} \in \Omega_x,$  be the maximizer of $H$ according to Lemma 4 and define $\mathcal{N}^{(*)}_k = \{\Delta \vect{x}^{(*)}_{k}, \Omega_{\theta}\}.$ Further, let $\mathcal{M}^{(*)}_k \cup \mathcal{N}^{(*)}_k$ be nonempty according to Lemma.~6,  then $(\Delta \vect{x}^{(*)}_{k},\vect{\theta}^{(*)}_{k}) \in \mathcal{M}^{(*)}_k \cup \mathcal{N}^{(*)}_k$ is a local equilibrium point.
\end{thm}

\begin{proof}
	By Lemma \ref{lem:lem_min_opt} we have at $ (\Delta \vect{x}_k^{ (* )}, \vect{\theta}^{ (* )}_k ),  (\Delta \vect{x}_k^{ (* )}, \vect{\theta}^{ (i )}_k ) \in \mathcal{M}^{ (* )}_k \cup \mathcal{N}^{ (* )}_k$ that
	\begin{equation}
		\begin{aligned}
			H (\Delta \vect{x}_k^{ (* )}, \vect{\theta}^{ (* )}_k ) &\leq& H (\Delta \vect{x}_k^{ (* )}, \vect{\theta}^{ (i )}_k ).
		\end{aligned}
	\end{equation}
	Similarly, according to Lemma \ref{lem:lem_max_opt}, at $ (\Delta \vect{x}^{ (* )}_k, \vect{\theta}^{ (* )}_{k} ),  (\Delta \vect{x}^{ (i )}_k, \vect{\theta}^{ (* )}_{k} ) \in \mathcal{M}^{ (* )}_k \cup \mathcal{N}^{ (* )}_k$ we have
	\begin{equation}
		\begin{aligned}
			H (\Delta \vect{x}^{ (* )}_k, \vect{\theta}^{ (* )}_{k} ) \geq H (\Delta \vect{x}^{ (i )}_k, \vect{\theta}^{ (* )}_{k} ).
		\end{aligned}
	\end{equation}
	Putting these inequalities together, we get
	\begin{equation}
		\begin{aligned}
			H (\Delta \vect{x}_k^{ (* )}, \vect{\theta}^{ (i )}_k ) \geq H (\Delta \vect{x}_k^{ (* )}, \vect{\theta}^{ (* )}_k ) \geq H (\Delta \vect{x}_k^{ (i )}, \vect{\theta}^{ (* )}_k ),
		\end{aligned}
	\end{equation}
	which is the saddle point condition, and therefore $ (\Delta \vect{x}_k^{   (*  )}, \vect{\theta}^{(*)}_k )$ is a local equilibrium point in $\mathcal{M}^{ (* )}_k \cup \mathcal{N}^{ (* )}_k.$ since  $\mathcal{M}^{(*)}_k \cup \mathcal{N}^{(*)}_k$ be nonempty according to Lemma.~6.
\end{proof}

We next show that this equilibrium point is stable~(Theorem~2) under a sequential play. Specifically, we show that when player 1 plays first and player~2 plays second, we asymptotically reach a saddle point pair $(\Delta \vect{x}^{(*)}_{k},\vect{\theta}^{(*)}_{k})$ for $H.$ At this saddle point, both players have no incentive to move, and the game converges. 
\begin{thm}[Stability of the Equilibrium Point]
For any $k \in [0,K]$, $\Delta \vect{x}^{(i)}_{k} \in \Omega_x$ and $\vect{\theta}^{(i)}_{k} \in \Omega_\theta$ be the initial values for $\Delta \vect{x}^{(i)}_{k}$ and $\vect{\theta}^{(i)}_{k}$ respectively. Define $\mathcal{M}_k = \{\Omega_{x}, \Omega_{\theta}\}$  with $H(\Delta \vect{x}^{(i)}_{k}, \vect{\theta}^{(i)}_{k})$ given by Proposition 2. Let $\Delta \vect{x}^{(i+1)}_{k} - \Delta \vect{x}^{(i)}_{k} = \alpha_{k}^{(i)}\times (\nabla_{\Delta \vect{x}^{(i)}_{k}} H(\Delta \vect{x}^{(i)}_{k}, \vect{\theta}^{(.)}_{k}) )/\| \nabla_{\Delta \vect{x}^{(i)}_{k}} H(\Delta \vect{x}^{(i)}_{k}, \vect{\theta}^{(.)}_{k}) \|^2)$ and  $\vect{\theta}^{(i+1)}_{k} - \vect{\theta}^{(i)}_{k} = -\alpha_{k}^{(i)}\times \nabla_{{\vect{\theta}}^{(i)}_{k}} H(\Delta \vect{x}^{(.)}_{k}, \vect{\theta}^{(i)}_{k}).$ Let the existence of an equilibrium point be given by Theorem~1, then, as a consequence of Lemmas 2 and 3, $(\Delta \vect{x}^{(*)}_{k},\vect{\theta}^{(*)}_{k}) \in \mathcal{M}_k$ is a stable equilibrium point for  $H$ given
 \label{thm:thm_st}.
\end{thm}
\begin{proof}
	Consider now the order of plays by the two players. By Lemma~\ref{lem:lem_max}, a game starting at $ (\Delta \vect{x}_k^{ (i )}, \vect{\theta}^{ (i )}_k ) \in \mathcal{M}_k$ will reach $ (\Delta \vect{x}_k^{ (* )}, \vect{\theta}^{ (i )}_k )$ which is a maximizer for $H.$ Now, define $\mathcal{N}_k =  (\Delta \vect{x}_k^{ (* )}, \Omega_{\theta}  ) \subset \mathcal{M}_k$ then a game starting at $ (\Delta \vect{x}_k^{ (* )}, \vect{\theta}^{ (i )}_k ) \in \mathcal{N}_k$ will converge to  $ (\Delta \vect{x}_k^{ (* )}, \vect{\theta}^{ (* )}_k ) \in \mathcal{N}_k$ according to Lemma~\ref{lem:lem_min}. Since,  $\mathcal{N}_k \subset \mathcal{M}_k,$ our result follows.
\end{proof}
In this game, the interplay between these two opposing players~(representative of generalization and forgetting, respectively)  introduces the dynamics required to play the game. Furthermore, the results presented in this section are local to the task. In other words, we prove that we can achieve a balance between generalization and forgetting for each task $k$~(neighborhoods are task dependent, and we achieve a local solution given a task $k$). 
Furthermore, our game is sequential; that is,  there is a leader~(player 1) and  a follower~(player 2). The leader~$(\Delta \vect{x}^{(i)}_{k})$ plays first, and  the follower~$(\vect{\theta}^{(i)}_{k})$ plays second with complete knowledge of the leader's play. The game is directed by $\Delta \vect{x}^{(i)}_{k}$, and any changes in the task~(reflected in $\Delta \vect{x}^{(i)}_{k})$ will shift the input and thus the equilibrium point. Consequently, the equilibrium point varies with respect to a task, and one will need to attain a new equilibrium point for each shift in a task. Without complete knowledge of the tasks~(not available in a CL scenario), only a local result is possible. This highlights one of the key limitations of this work. Ideally, we would like a balance between forgetting and generalization that is independent of tasks. However, this would require learning a trajectory of the equilibrium point~(How does the equilibrium point change with the change in the tasks?) and is beyond the scope of this paper. One work that attempts to do this is~\cite{rusu2016progressive}, where the authors learn a parameter per task. For a large number of tasks, however, such an approach is computationally prohibitive.

These results are  valid only under certain assumptions: (1) the Frobenius norm of the gradient is bounded, always positive; (2) the cost function is bounded and differentiable; and (3) the learning rate goes to zero as $i$ tends to infinity. The first assumption is reasonable in practice, and gradient clipping or perturbation strategies can be used to ensure it. The boundedness of the cost  (second assumption) can be ensured by  prioritizing the contributions of the task~(Lemma~1 and Corollary~1). The third assumption assumes a decaying learning rate. Learning rate decay is a common strategy and is employed widely. Therefore, all assumptions are practical and reasonable.

\begin{wrapfigure}[19]{r}{0.44\columnwidth}
	\vspace{-11mm}
	\begin{algorithm}[H]
		\SetCustomAlgoRuledWidth{0.40\columnwidth}
		Initialize ${\vect{\theta}}, D_{P}, D_{N}$ \\
		\While{$k=1,2,3,... K$}{
			j = 0\\
			\While{$j < \rho$}{
				Get $\vect{b}_{N} \in D^{N}_{k}$
				
				Get $\vect{b}_{P} \in D^{P}_{k}$
				
				Get $\vect{b}_{PN} = \vect{b}_{P} \cup \vect{b}_{N}$ 
				
				Copy $\vect{b}_{PN}$ into $\vect{x}^{PN}_{k}$
				
				i = 0
				\While{$i+1 <= \zeta$}{
					Update~$\vect{x}^{PN}_{k}$ with $J_{k}({\vect{\theta}}_k)$ using gradient ascent
					
					i = i+1 }
				Calculate $J_{k+\zeta}(\vect{\theta}^{(i)}_{k}) - J_{k}(\vect{\theta}^{(i)}_{k})$ 
				
				Copy ${\vect{\theta}}^{(i)}_k$ into ${\vect{\theta}}^{B}_k$
				
				i = 0
				\While{$i+1 <= \zeta$}{
					Update~${\vect{\theta}}^{B}_{k}$ with $J_{k}(\vect{\theta}^{B}_{k})$
					
					i = i+1 }
				Calculate $(J_{k}({\vect{\theta}}^{B}_{k}) - J_{k}(\vect{\theta}^{(i)}_{k}))$
				
				Calculate $ H(\Delta \vect{x}^{(i)}_{k}, \vect{\theta}^{(i)}_{k})$
				
				Update ${\vect{\theta}}^{(i)}_k$ using gradient descent
				
			}
			j= j+1
		}
		Update ${D_{P}}$ with  ${D_{N}}$
		\caption{BCL \label{alg1a}}
	\end{algorithm}
\end{wrapfigure} 
\subsection{Balanced Continual Learning~\label{BCL}}  Equipped with the theory, we develop a new CL method to achieve a balance between forgetting and generalization. By Proposition~2, the cost function can be upper bounded as $	H(\Delta \vect{x}^{(i)}_{k}, \vect{\theta}^{(i)}_{k}) \leq \beta_k  J_{k}(\vect{\theta}^{(i)}_{k}) + (J_{k}(\vect{\theta}^{(i+\zeta)}_{k}) - J_{k}(\vect{\theta}^{(i)}_{k})) + ( J_{k+\zeta}(\vect{\theta}^{(i)}_{k}) - J_{k}(\vect{\theta}^{(i)}_{k}) ),$ where $J_{k+\zeta}$ indicates $\zeta$ updates on player 1 and $\vect{\theta}^{(i+\zeta)}_{k}$ indicates $\zeta$ updates on player 2.

    \begin{equation}
	\begin{aligned}
         \underbrace{\frac{\alpha_{k}^{(i)} \nabla_{\Delta \vect{x}_{k}} E[ H(\Delta \vect{x}^{(i)}_{k}, \vect{\theta}^{(i)}_{k})] }{\| \nabla_{\Delta \vect{x}_{k}}  H(\Delta \vect{x}^{(i)}_{k}, \vect{\theta}^{(i)}_{k}) \|^2}}_{\text{Player~1}}, \\
         \underbrace{-\alpha_{k}^{(i)}\times \nabla_{{\vect{\theta}}_{k}} E[H(\Delta \vect{x}^{(*)}_{k}, \vect{\theta}^{(i)}_{k}) )]}_{\text{Player~2}},
    	\end{aligned}
    	\label{eq:eq_Strat}
    \end{equation} 

The strategies for the two players $\Delta \vect{x}_{k}, {\vect{\theta}}_{k}$ are chosen in Eq. \eqref{eq:eq_Strat} with $E$ being the expected value operator. We can approximate the required terms in our update rule~(player strategies) using data samples (batches). Note that the approximation is  performed largely through one-sided finite difference, which may introduce an error and is another potential drawback. The pseudo code of the BCL is shown in Algorithm~\ref{alg1a}. We define a new task array $\cD_N(k)$ and a task memory array~$\cD_{P}(k) \subset \cup_{\tau = 0}^{k-1}\cT_{\tau}$~(samples from all previous tasks). For each batch $b_{N} \in \cD_N(k)$, we sample $b_{P}$ from~$\cD_P(k),$ combine to create $b_{PN}(k)= b_{P}(k) \cup b_{N}(k)$, and perform a sequential play. Specifically, for each task the first player initializes $x_k^{PN} = b_{PN}(k)$ and performs  $\zeta$  updates on $x_k^{PN}$ through gradient ascent. The second player, with complete knowledge of the first player's strategy, chooses the best play to reduce $H(\Delta \vect{x}^{(i)}_{k}, \vect{\theta}^{(i)}_{k})$. To estimate player~2's play, we must estimate different terms in $H(\Delta \vect{x}^{(i)}_{k}, \vect{\theta}^{(i)}_{k}).$ This procedure involves  three steps. First, we use the first player's play and approximate $( J_{k+\zeta}(\vect{\theta}^{(i)}_{k}) - J_{k}(\vect{\theta}^{(i)}_{k}) ).$ Second, to approximate $(J_{k}(\vect{\theta}^{(i+\zeta)}_{k}) - J_{k}(\vect{\theta}^{(i)}_{k})):$ (a), we copy $\hat{\vect{\theta}}$ into $\hat{\vect{\theta}}_{B}$~(a temporary network) and perform $\zeta$ updates on $\hat{\vect{\theta}}_{B}$; and (b) we compute $J_{k}(\vect{\theta}^{(i+\zeta)}_{k})$ using $\hat{\vect{\theta}}_{B}(k+\zeta)$ and evaluate $(J_{k}(\vect{\theta}^{(i+\zeta)}_{k}) - J_{k}(\vect{\theta}^{(i)}_{k})).$ Third, equipped with these approximations, we compute $H(\Delta \vect{x}^{(i)}_{k}, \vect{\theta}^{(i)}_{k})$ and obtain the play for the second player. Both these players perform the steps repetitively for each piece of information~(batch of data). Once all the data from the new task is exhausted, we  move to the next task. 

\subsection{Related Work}

Traditional solutions to the CL focus on either the  forgetting issue~\cite{rusu2016progressive, yoon2017lifelong, yao2020don, caccia2021online, kirkpatrick2017overcoming, zenke2017continual, aljundi2018memory, lin1992self,  lopez2017gradient, chaudhry2019continual} or the generalization issue~\cite{vinyals2016matching, snell2017prototypical, finn2017model, finn2019online, Caccia2020OnlineFA}. Common solutions to the forgetting problem involve  dynamic architectures and flexible knowledge representation such as \cite{rusu2016progressive, yoon2017lifelong, yao2020don, caccia2021online}, regularization approaches including \cite{kirkpatrick2017overcoming, zenke2017continual, aljundi2018memory} and memory/experience replay \cite{lin1992self,  lopez2017gradient, chaudhry2019continual}. 
Similarly, quick generalization to a new task has been addressed through few-shot and one-shot learning approaches such as matching nets \cite{vinyals2016matching} and  prototypical network \cite{snell2017prototypical}. More recently, the field of metalearning has approached the generalization problem by designing a metalearner that can perform quick generalization from very little data \cite{finn2017model, finn2019online, Caccia2020OnlineFA}. 

In the past few years, metalearners for quick generalization have been combined with methodologies specifically designed for reduced forgetting \cite{javed2019meta, beaulieu2020learning}. For instance, the approaches in \cite{javed2019meta, beaulieu2020learning} adapt the model-agnostic metalearning~(MAML) framework in \cite{finn2017model}  with robust representation to minimize forgetting and provide impressive results on CL. However, both these approaches require a pretraining phase for learning representation.  Simultaneously, Gupta et al.~\cite{gupta2020lamaml} introduced LA-MAML---a metalearning approach where the impact of learning rates on the CL problem is reduced through the use of per-parameter learning rates. LA-MAML~\cite{gupta2020lamaml} also introduced episodic memory to address the forgetting issue. Other approaches also have attempted to model both generalization and forgetting. In \cite{farajtabar2019orthogonal}, the gradients from new tasks are projected onto a subspace that is orthogonal to the older tasks, and forgetting is minimized. Similarly, Joseph and Balasubramanian \cite{joseph2020metaconsolidation} utilized a Bayesian framework to consolidate learning across previous and current tasks, and  Yin et  al.~\cite{yin2020optimization} provided a framework for approximating loss function to summarize the forgetting in the CL setting.  Furthermore, Abolfathi et al.~\cite{abolfathi2021coachnet} focused on sampling episodes in the reinforcement learning setting, and Elrahimi et al.~\cite{ebrahimi2020adversarial} introduced a generative adversarial network-type structure to progressively learn shared features assisting reduced forgetting and improved generalization. Despite significant progress, however, these methods~\cite{javed2019meta, beaulieu2020learning, gupta2020lamaml,farajtabar2019orthogonal, joseph2020metaconsolidation, yin2020optimization, abolfathi2021coachnet, ebrahimi2020adversarial} are still inherently tilted toward maximizing generalization or minimizing forgetting because  they naively minimize the loss function. Therefore, the contribution of different terms in the loss function becomes important. For instance, if the generalization cost is given more weight, a method would generalize better. Similarly, if forgetting cost is given more weight, a method would forget less. Therefore, the resolution of the trade-off inherently depends on an hyperparameter.

The first work to formalize the trade-off in CL was MER, where the trade-off was formalized as a gradient alignment problem. Similar to MER, Doan et al.~(\cite{doan2021theoretical}) studied forgetting as an alignment problem. In MER, the angle between the gradients was approximated by using Reptile \cite{nichol2018firstorder}, which promotes gradient alignment by reducing weight changes. On the other hand, \cite{doan2021theoretical}  formalized the alignment as an eigenvalue problem and introduced a PCA-driven method to ensure alignment. Our approach models this balance as a saddle point problem achieved through stochastic gradient such that the saddle point~(balance point or equilibrium point) resolves the trade-off.

Our approach is the first in the CL literature to prove the existence of the saddle point~(the balance point) between generalization and forgetting given a task in a CL problem. Furthermore, we are  the first to theoretically demonstrate that the saddle point can be achieved reasonably under a gradient ascent-descent game. The work closest to ours is \cite{ebrahimi2020adversarial}, where an adversarial framework is described to minimize forgetting in CL by generating task-invariant representation. However, \cite{ebrahimi2020adversarial} is not model agnostic~(the architecture of the network is important) and requires a considerable amount of data at the start of the learning procedure. Because of these issues, \cite{ebrahimi2020adversarial} is not suitable for learning in the sequential scenario.

\section{Experiments}
We use the CL benchmark \cite{Hsu18_EvalCL} for our experiments and retain the experimental settings~(hyperparameters) from \cite{Hsu18_EvalCL, vandeven2019generative}. For comparison, we use the split-MNIST, permuted-MNIST, and split-CiFAR100 data sets while considering three scenarios: incremental domain learning~(IDL), incremental task learning~(ITL), and incremental class learning~(ICL). The splitting and permutation strategies when applied to the MNIST or CiFAR100 data set can generate task sequences for all  three scenarios~(illustrated in Figure~1 and Appendix: Figure~2 of \cite{Hsu18_EvalCL}). For comparing our approach, we use three baseline strategies---standard neural network with Adam~\cite{kingma2014adam}, Adagrad~\cite{duchi2011adaptive}, and SG---and use $L_2$-regularization and naive rehearsal~(which is similar to experience replay). For CL approaches, we use EWC~\cite{kirkpatrick2017overcoming}, online EWC~\cite{schwarz2018progress}, SI~\cite{zenke2017continual}, LwF~\cite{LwF},  DGR~\cite{DGR}, RtF~\cite{vandeven2019generative}, MAS~\cite{aljundi2018memory}, MER~\cite{riemer2018learning}, and GEM~\cite{GEM}. We utilize data preprocessing as provided by \cite{Hsu18_EvalCL}. Additional details on experiments can be found in  Appendix~B and \cite{Hsu18_EvalCL, vandeven2019generative}. All experiments are conducted in Python 3.4 using the pytorch $1.7.1$ library with the NVIDIA-A100 GPU for our simulations.
\begin{table}[bt]
\vspace{-5pt}
\tiny
    \centering
    \resizebox{\columnwidth}{!}{ 
    \begin{tabular}{c|ccc|ccc}
    \hline \hline
                         \multirow{3}{*}{\textbf{Method}}&\multicolumn{3}{c|}{\textbf{split-MNIST}}  &
                           \multicolumn{3}{c}{\textbf{permuted-MNIST}} \\ \cline{2-7}
                          & Incremental          & Incremental        & Incremental     
                          & Incremental          & Incremental        & Incremental\\
                          & task learning        & domain learning    & class learning  
                          & task learning        & domain learning    & class learning\\
                          & [ITL]                & [IDL]              & [ICL]        
                          & [ITL]                & [IDL]              & [ICL]        \\ \hline
                      Adam& $95.52 \pm 2.14$     & $54.75 \pm2.06$    &$19.72\pm0.03$
                          & $93.42 \pm 0.56$     & $77.87 \pm1.27$    &$14.02\pm1.25$\\
                       SGD& $97.65 \pm0.28$      & $62.80 \pm0.34$    &$19.36\pm0.02$
                          & $90.95 \pm0.20$      & $78.17 \pm1.16$    &$12.82\pm0.95$\\
                   Adagrad& $98.37 \pm0.29$      & $57.59 \pm2.54$    &$19.59\pm0.17$
                          & $92.45 \pm0.16$      & $91.59 \pm0.46$    &$29.09\pm1.48$\\
                     $L_2$& $97.62 \pm0.69$      & $66.84 \pm3.91$.   &$22.92\pm1.90$
                          & $94.87 \pm0.38$      & $92.81 \pm0.32$    &$13.92\pm1.79$\\
           Naive rehearsal& $99.32 \pm0.10$      & $94.85 \pm0.80$    &$90.88\pm0.70$
                          & $96.23 \pm0.04$      & $95.84 \pm0.06$    &$96.25\pm0.10$\\
         Naive rehearsal-C& $99.41 \pm0.04$      & $97.13 \pm0.37$    &$94.92\pm0.63$
                          & $97.13 \pm0.03$      & $96.75 \pm0.03$    &$97.24\pm0.05$\\
                          \hline \hline  
                       EWC& $96.59 \pm0.99$       &$57.31\pm1.07$     &$19.70\pm0.14$
                          & $95.38 \pm0.33$       &$89.54\pm0.52$     &$26.32\pm4.32$\\
                Online~EWC& $99.01 \pm0.12$       &$58.25\pm1.23$     &$19.68\pm0.05$
                          & $95.15 \pm0.49$       &$93.47\pm0.01$     &$42.58\pm6.50$\\
                        SI& $99.10 \pm0.16$       &$64.63\pm1.67$     &$19.67\pm0.25$
                          & $94.35 \pm0.51$       &$91.12\pm0.93$     &$58.52\pm4.20$\\ 
                       MAS& $98.88 \pm0.14$       &$61.98\pm7.17$     &$19.70\pm0.34$
                          & $94.74 \pm0.52$       &$93.22\pm0.80$     &$50.81\pm2.92$\\
                       GEM& $98.32 \pm0.08$       &$97.37\pm0.22$     &$93.04\pm0.05$
                          & $95.44 \pm0.96$       &$96.86\pm0.02$     &$96.72\pm0.03$\\
                       DGR& $99.47 \pm0.03$       &$95.74\pm0.23$     &$91.24\pm0.33$
                          & $92.52 \pm0.08$       &$95.09\pm0.04$     &$92.19\pm0.09$\\
                       RtF&$\textbf{99.66}\pm\textbf{0.03}$&$97.31\pm0.11$&$92.56\pm0.21$
                          &$97.31\pm0.01$        &$97.06\pm0.02$      &$96.23\pm0.04$\\ 
                       MER&$97.12\pm0.10$        &$92.16\pm0.35$      &$93.20\pm0.12$
                          &$97.15\pm0.08$        &$96.11\pm0.31$      &$91.71\pm0.03$\\ \cline{2-7}  
          BCL~(With Game) &$99.52\pm0.07$                           & $\textbf{98.71 }\pm\textbf{0.06}$     &$\textbf{97.32}\pm\textbf{0.17}$
                          & $\textbf{97.41} \pm \textbf{0.01}$      & $\textbf{97.51} \pm \textbf{0.05}$    &$\textbf{97.61}\pm\textbf{0.01}$ \\
      BCL~(Without Game)  & $97.73 \pm 0.03$                        & $96.43 \pm0.29$                       &$91.88\pm0.55$
                          & $96.16 \pm 0.03$                        & $96.08 \pm 0.06$                      &$95.96\pm0.06$\\
       \hline
       \hline
    \end{tabular}}
    \caption{Performance of our approach  compared with other methods in the literature. We record the mean and standard deviation of the retained accuracy for the different methods. The best scores are in bold.}
    \label{tab:main_res}
    \vspace{-8mm}
\end{table}

\textbf{Comparison with the state of the art:} The results for our method are summarized in Table~\ref{tab:main_res} and \ref{tab:cifar_100}.  The efficiency for any method is calculated  by observing the  average accuracy~(retained accuracy (RA)~\cite{riemer2018learning}) at the end of each repetition and then evaluating the mean and standard deviation of RA across different repetitions. For each method, we report the mean and standard deviation of RA over five repetitions of each experiment. In each column, we indicate the best-performing method in bold. For the split-MNIST data set, we obtain $99.52\pm0.07$ for ITL, $98.71\pm0.06$ for IDL, and $97.32\pm 0.17$ for ICL. Similarly, with the permuted-MNIST data set, we obtain $97.41\pm0.01$ for ITL, $97.51\pm0.05$ for IDL, and $97.61\pm0.01$ for ICL. Furthermore, with the split-CiFAR100 data set, we obtain $81.82 \pm 0.17$  for ITL,  $62.11\pm 0.00$  for IDL, and  $69.27 \pm 0.03$ for ICL. BCL is the best-performing methodology for all cases~(across both data sets) except RtF for ITL~($0.14 \%$ drop) with the split-MNIST data set.

\begin{wraptable}[16]{r}{0.61\columnwidth}
    \tiny
    \centering
    \begin{tabular}{c|ccc}
    \hline \hline
                         \multirow{3}{*}{\textbf{Method}}&\multicolumn{3}{c}{\textbf{split-CiFAR100}}\\ \cline{2-4}
                          & Incremental          & Incremental        & Incremental  \\
                          & task learning        & domain learning    & class learning \\ \hline
                      Adam& $30.53 \pm0.58$     & $19.65  \pm0.14$    &$17.20\pm0.06$\\
                       SGD& $43.77 \pm1.15$      & $19.17 \pm0.12$    &$17.18\pm0.12$\\
                   Adagrad& $36.27 \pm0.43$      & $19.06 \pm0.14$    &$15.83\pm0.20$\\
                     $L_2$& $51.73 \pm1.30$      & $19.96 \pm0.15$    &$17.12\pm0.04$\\
           Naive rehearsal& $70.20 \pm0.17$      & $35.94 \pm0.39$    &$34.33\pm0.19$\\
         Naive rehearsal-C& $78.41 \pm0.37$      & $51.81 \pm0.18$    &$51.28\pm0.17$\\
                          \hline \hline  
                       EWC& $61.11 \pm1.43$       &$19.76\pm0.12$     &$19.70\pm0.14$\\
                Online~EWC& $63.22 \pm0.97$       &$20.03\pm0.10$     &$17.16\pm0.09$\\
                        SI& $64.81 \pm1.00$       &$20.26\pm0.09$     &$17.26\pm0.11$\\ 
                       MAS& $64.77 \pm0.78$       &$19.99\pm0.16$     &$17.07\pm0.12$\\ \cline{2-4}
        BCL(With Game)    & $\textbf{81.82} \pm \textbf{0.17}$        &$\textbf{62.11 }\pm\textbf{0.00}$  &  $\textbf{69.27} \pm \textbf{0.03}$ \\
        BCL(Without Game) & $69.17 \pm0.12$       &$51.82\pm0.19$     & $52.82 \pm 0.01$   \\
       \hline
       \hline
    \end{tabular}
    \caption{Performance of BCL for the split-CiFAR100 data set. We record the retained accuracy for the different methods. We obtained RA scores for all methods except BCL from \cite{Hsu18_EvalCL}.}
    \label{tab:cifar_100}
\end{wraptable} Generally, ITL is the easiest learning scenario~\cite{Hsu18_EvalCL}, and all methods therefore perform  well on ITL~(the performance is close). For the ITL scenario with the split-MNIST data set, BCL is better than most methods; but several methods, such as naive rehearsal, naive rehearsal-C, RtF, and DGR, attain close RA values~(less than $1\%$ from BCL). Note that both DGR and RtF involve a generative model pretrained with data from all the tasks. In a sequential learning scenario, one cannot efficiently train generative models because data corresponding to all the tasks is not available beforehand. Although RtF provides improved performance for split-MNIST~(ITL), the improvement is less than $1\%$~(not significant).  In fact, RtF performance is poorer for BCL in ICL by $4.76\%$~(a significant drop in performance) and in IDL  by $1.4 \%.$  

 Two additional observations can be made about the split-MNIST data set. First, Adagrad, SGD, and L2 achieve better performance than does Adam. Therefore, in our analysis Adagrad appears more appropriate although Adam is popularly used for this task.  Second, naive rehearsal~(both naive rehearsal and naive rehearsal-C approaches) achieves performance equivalent to the state-of-the-art methods with similar memory overhead. Furthermore, naive rehearsal performs much better than online EWC and SI, especially in the ICL scenario. These limitations indicate that  regularization-driven approaches are not much better than baseline models and in fact perform poorer than methods involving memory~(naive rehearsals, MER, BCL, etc.). In~\cite{knoblauch2020optimal}, it was shown theoretically that memory-based approaches typically do better than regularization-driven approaches, as is empirically observed in this paper, too. Another interesting observation is that EWC and online EWC require significant hyperparameter tuning, which would be difficult to do in real-world scenarios. Other regularization-based methods, such as SI and MAS, also suffer from the same issue.

The observations from the split-MNIST carry forward to the permuted-MNIST data set. Moreover, RA values for the permutation MNIST data set are better for the split-MNIST data set across the board, indicating that the permutation MNIST data set presents an easier learning problem. Similar to the observations made with the split-MNIST data set, BCL is better than all methods for the permuted-MNIST dataset, with naive rehearsal and RtF providing RA values that are close~(less than $1 \%$). The only methodology in the literature that attempts to model the balance between forgetting and generalization is MER, an extension of GEM. From our results, we observe that BCL is better than MER in all cases~(Split-MNIST--$2.4 \%$ improvement for ITL, $6.55 \%$ improvement for IDL, $4.12 \%$ improvement for ICL and Permuted-MNIST--$0.26 \%$ improvement for ITL, $1.4  \%$ improvement for IDL and $5.9 \%$ improvement for ICL).

The substantial improvements obtained by BCL are more evident in Table \ref{tab:cifar_100} where the results on the split-CiFAR100 data set are summarized. BCL is clearly the best-performing method. The next best-performing method is naive rehearsal-C, where BCL improves performance by $3.41 \%$ for ITL, $10.3 \%$ for IDL, and  $17.99 \%$ for ICL. Other observations about the regularization methods and the rest of the baseline methods carry forward from Table~\ref{tab:main_res}. However, one key difference is that while AdaGrad  is observed to be better than Adam for the MNIST data set, Adam is comparable to SGD and Adagrad for split-CiFAR100. In summary, BCL is comparable to or better than the state of the art in the literature for both the MNIST and CiFAR100 data sets.

\begin{wrapfigure}[30]{l}{0.49\textwidth}
\begin{subfigure}{\textwidth}
  \centering
	\includegraphics[width = \columnwidth]{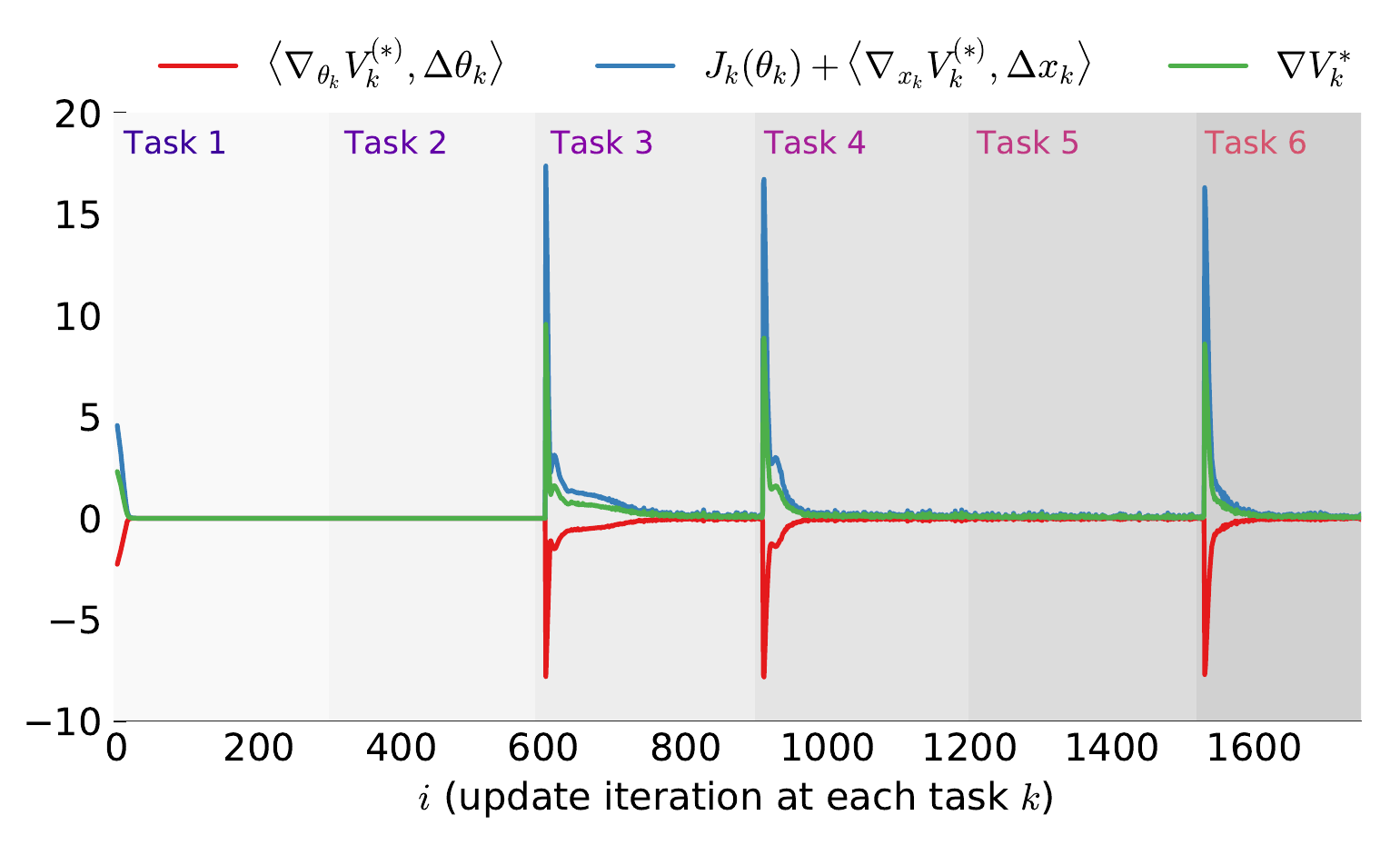}
	\caption{}
  \label{fig:all_task}
\end{subfigure}\\
\begin{subfigure}{\textwidth}
 \centering
	\includegraphics[width = \columnwidth]{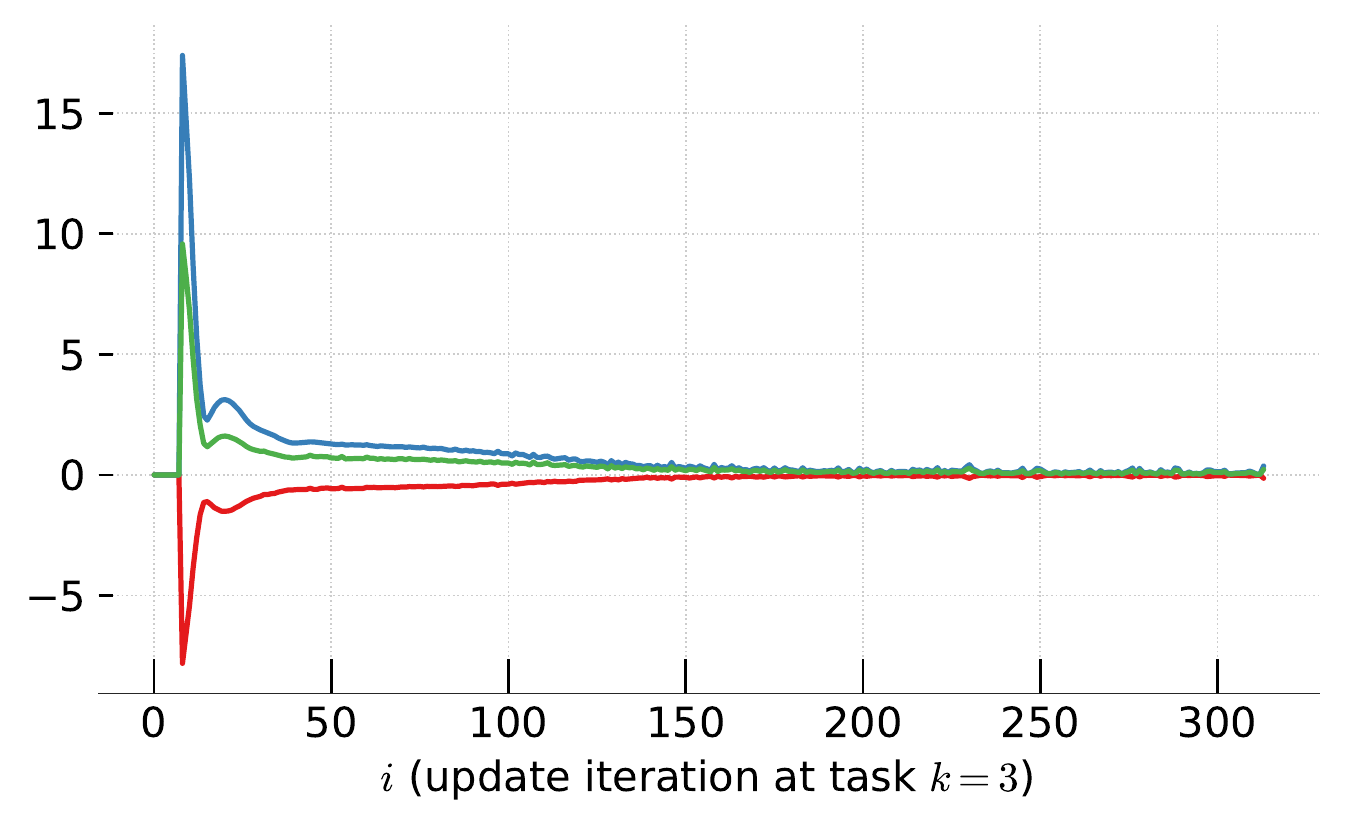}
	 \caption{}
	 \label{fig:one_task}
\end{subfigure}
\caption{(Top) Progression of different terms in Eq. \eqref{eq_M_DES} with respect to update iterations, where the index on the x axis is calculated as $k\times 300$. The tasks boundaries~(at what $i$ the tasks are introduced) are illustrated through shades of grey. (down) Illustration of the cost at $k =3$.}
\end{wrapfigure}
                

\textbf{Do we need a game to achieve this performance?} To provide additional insights into the benefits of the game, we compare RA values with and without the game. In our setup, $\Delta \vect{x}^{(i)}_{k}$ aims at increasing the cost, and  $\vect{\theta}^{(i)}_{k}$ aims at reducing the cost. If we hold the play for $\Delta \vect{x}^{(i)}_{k}$, then the dynamics required to play the game do not exist, thus providing a method that can perform CL without the game. Therefore, we induce the absence of a game by fixing $\Delta \vect{x}^{(i)}_{k}$ and perform each of the nine experiments for five repetitions~(ICL, IDL, and ITL for split-MNIST, permuted-MNIST, and split-CiFAR100). We summarize these results in the last two rows of Tables \ref{tab:main_res} and \ref{tab:cifar_100}. Consequently, we make two observations. First: even without the game, we achieve RA values comparable to the state of the art. This is true for both the MNIST and CiFAR100 data sets. Second, with the introduction of the game, we observe improved RA values across the board~(at least by $1 \%$ with MNIST). The difference is clearer with the CiFAR100 data set where we observe a substantial improvement in RA values~(at least $10 \%$).

\textbf{Does the theory appropriately model the continual learning problem?} In Fig.~\ref{fig:all_task} we plot the progression of $J_{k}( {\vect{\theta}}_{k})+ \langle \nabla_{\vect{x}_{k}} V_{k}^{(*)}, \Delta \vect{x}_{k}\rangle$~(blue curve), $\langle \nabla_{{\vect{\theta}}_{k}} V_{k}^{(*)}, \Delta {\vect{\theta}}_{k} \rangle$~(red curve), and the sum, namely, $\Delta V^{(*)}_{k} $~(green curve). A total of six tasks,  sampled from the permutation MNIST data set within the ICL setting, are illustrated in Fig. \ref{fig:all_task}. These tasks are introduced every 300 update steps. From Fig. \ref{fig:all_task} we make two important observations. First, as soon as tasks 1, 3, 4, and 6 are introduced, all three curves~(red, blue, and green) indicate a bump. Second, when tasks 2 and 5 are introduced, there is no change. On a closer look at task 3 in Fig.~\ref{fig:one_task}, we observe that when the task 3 is introduced, the blue curve exhibits a large positive bump~(the introduction of the new tasks increases the first and the third terms in Eq. \eqref{eq_M_DES}). The increase implies that task 3 forced the model to generalize and increased forgetting on tasks 1 and 2~(observed by the increase in the green curve). To compensate for this increase, we require the model~$\vect{\theta}_k$ to behave adversarially and introduce a large enough negative value in $\langle \nabla_{{\vect{\theta}}_{k}} V_{k}^{(*)}, \Delta {\vect{\theta}}_{k} \rangle$~(red curve)  to cancel out the increase in the blue curve. In Fig.~\ref{fig:one_task} the red curve demonstrates a large negative value~(expected behavior) and eventually~(as $i$ increases) forces the blue curve~(by consequence, green: the sum of red and blue) to move toward zero~(the model compensates for the increase in forgetting). As observed, the blue and the red curves behave opposite to each other and introduce a push-pull behavior that  stops   only when the two cancel each other and the sum~(green) is zero. Once the sum has reached zero, there is no incentive for the red and green to be nonzero, and therefore they remain at zero; thus, all three curves (green, red, and blue) remain at zero once converged until a task 4 is introduced~(when there is another bump, as seen in Fig. \ref{fig:all_task}). However, this increase in the blue curve is not observed when tasks 2 and 5 are introduced. When new tasks are similar to the older tasks, it is expected that $\Delta V^{(*)}_{k} =0,$  as is observed in Fig. \ref{fig:all_task}.

All of these observations are fully explained by  Eq. \ref{eq_M_DES}, which illustrates that the solution to the CL problem is obtained optimally only when $\Delta V^{(*)}_{k} =0$~(observed in Figs. \ref{fig:all_task} and \ref{fig:one_task}). The term  $\Delta V^{(*)}_{k} $ is quantified by $J_{k}( {\vect{\theta}}_{k}), \langle \nabla_{{\vect{\theta}}_{k}} V_{k}^{(*)}, \Delta {\vect{\theta}}_{k} \rangle$ and $\langle \nabla_{\vect{x}_{k}} V_{k}^{(*)}, \Delta \vect{x}_{k}\rangle$. Our theory suggests that there exists an inherent trade-off between different terms in Eq. \ref{eq_M_DES}. Therefore every time a new task is observed, it is expected that~$J_{k}( {\vect{\theta}}_{k})+ \langle \nabla_{\vect{x}_{k}} V_{k}^{(*)}, \Delta \vect{x}_{k}\rangle$ increases~(increase in blue curve when tasks 1, 3, 4, and 6 are introduced) and $\langle \nabla_{{\vect{\theta}}_{k}} V_{k}^{(*)}, \Delta {\vect{\theta}}_{k} \rangle$ compensates to cancel this increase~(red curve exhibits a negative jump). In Theorems~1 and 2 we demonstrate the existence of this balance point~(for each task, as $i$ increases, $\Delta V^{(*)}_{k} $ tends to zero, as observed in Fig. \ref{fig:one_task}) and $\Delta V^{(*)}_{k} $  remains zero~(the balance point is stable, proved in Theorem~2)  until a new task increases forgetting. Furthermore, our theory claims the existence of a solution with respect to each task. This is also observed in Fig.~\ref{fig:all_task} as, for each task, there is an increase in cost, and BCL quickly facilitates convergence. \emph{These observations indicate that our dynamical system in Eq. \ref{eq_M_DES}
accurately describes the dynamics of the continual learning problem}. Furthermore, the assumptions under which the theory is developed are practical and are satisfied by performing continual learning on the permuted MNIST problem.

\section{Conclusion}
We developed a dynamic programming-based framework to enable the methodical study of key challenges in CL. We show that an inherent trade-off between generalization and forgetting exists and must be modeled for optimal performance. To this end, we introduce a two-player sequential game that models the trade-off. We show in theory and simulation that there is an equilibrium point that resolves this trade-off~(Theorem~1) and that this saddle point can be attained~(Theorem~2). However, we observe that any change in the task modifies the equilibrium point. Therefore, a global equilibrium point between generalization and forgetting is not possible, and our results are  valid only in a neighborhood~(defined given a task). To attain this equilibrium point, we develop BCL and demonstrate state-of-the-art performance on a CL benchmark~\cite{Hsu18_EvalCL}. In the future, we will extend our framework for nonEuclidean tasks.

\section{Broader Impact}
\textit{Positive Impacts:} CL has a wide range of applicability. It helps avoids retraining, and it improves the learning efficiency of learning methods. Therefore, in science applications where the data is generated sequentially but the data distribution varies with time, our theoretically grounded method provides the potential for improved performance.
\textit{Negative Impacts:} Our theoretical framework does not have  direct adverse impacts. However, the potential advantages of our approach can improve the efficiency of adverse ML systems such as fake news, surveillance, and cybersecurity attacks. 

\begin{ack}
This work was supported by the U.S.\ Department of Energy, Office of Science, Advanced Scientific Computing Research, under Contract DE-AC02-06CH11357 and by a DOE Early Career Research Program award. We are grateful for the computing resources from the Joint Laboratory for System Evaluation  and Leadership Computing Facility at Argonne. We also are grateful to Dr. Vignesh Narayanan, assistant professor, University of South Carolina, and Dr. Marieme Ngom, Dr. Sami Khairy --postdoctoral appointees, Argonne National Laboratory, for their insights.
\end{ack}

\bibliographystyle{plain}
\bibliography{cdc.bib}
\section*{Checklist}


\begin{enumerate}
\item For all authors...
\begin{enumerate}
  \item Do the main claims made in the abstract and introduction accurately reflect the paper's contributions and scope?
    \answerYes{See Fig 1~(right)}
  \item Did you describe the limitations of your work?
    \answerYes{Refer Section~\ref{theory} and Section~\ref{BCL}}
  \item Did you discuss any potential negative societal impacts of your work?
    \answerYes{Refer Section Broader Impact}
  \item Have you read the ethics review guidelines and ensured that your paper conforms to them?
    \answerYes{Yes}
\end{enumerate}

\item If you are including theoretical results...
\begin{enumerate}
  \item Did you state the full set of assumptions of all theoretical results?
    \answerYes{Refer to the theorem statements in section. \ref{theory}}
	\item Did you include complete proofs of all theoretical results?
    \answerYes{Refer to proofs in Appendix A}
\end{enumerate}

\item If you ran experiments...
\begin{enumerate}
  \item Did you include the code, data, and instructions needed to reproduce the main experimental results (either in the supplemental material or as a URL)?
    \answerYes{Refer to supplementary materials}
  \item Did you specify all the training details (e.g., data splits, hyperparameters, how they were chosen)?
    \answerYes{Refer to Appendix C in supplementary materials and \cite{Hsu18_EvalCL}}
	\item Did you report error bars (e.g., with respect to the random seed after running experiments multiple times)?
    \answerYes{Five repetitions and mean and standard deviations are reported. Refer to Table. \ref{tab:main_res}}
	\item Did you include the total amount of compute and the type of resources used (e.g., type of GPUs, internal cluster, or cloud provider)?
    \answerYes{Refer to the Results section}
\end{enumerate}

\item If you are using existing assets (e.g., code, data, models) or curating/releasing new assets...
\begin{enumerate}
  \item If your work uses existing assets, did you cite the creators?
    \answerYes{We used the continual learning benchmark by \cite{Hsu18_EvalCL}; we cite their paper.}
  \item Did you mention the license of the assets?
    \answerYes{No licence information was provided by \cite{Hsu18_EvalCL}.}
  \item Did you include any new assets either in the supplemental material or as a URL?
    \answerNA{}
  \item Did you discuss whether and how consent was obtained from people whose data you are using/curating?
    \answerNA{}
  \item Did you discuss whether the data you are using/curating contains personally identifiable information or offensive content?
    \answerNA{}
\end{enumerate}

\item If you used crowdsourcing or conducted research with human subjects...
\begin{enumerate}
  \item Did you include the full text of instructions given to participants and screenshots, if applicable?
    \answerNA{}
  \item Did you describe any potential participant risks, with links to Institutional Review Board (IRB) approvals, if applicable?
  \answerNA{}
  \item Did you include the estimated hourly wage paid to participants and the total amount spent on participant compensation?
    \answerNA{}
\end{enumerate}
\end{enumerate}

\section*{Supplementary Information}
We use $\mR$ to denote the set of real numbers and $\mN$ to denote the set of natural numbers. We  use $\|.\|$ to denote the Euclidean norm for vectors and the Frobenius norm for matrices, while using bold symbols to illustrate matrices and vectors. We define an interval $[0, K ), K \in \mN$ and  let $p (\cT )$ be the distribution over all the tasks observed in this interval. For any $k \in [0, K ) ,$ we define a parametric model $g (. )$ with ${\vect{y}}_{k} = g (\vect{x}_{k}; {\vect{\theta}}_{k} )$, where ${\vect{\theta}}_{k}$ is a vector comprising all parameters of the model with $\vect{x}_{k} \in \cX_{k}$.  Let $n$ be the number of samples  and  $m$ be the number of dimensions. Suppose a task at $k | k \in [0, K )$ is observed and denoted as  $\cT_{k} : \cT_{k} \sim p (\cT )$, where $\cT_{k} =\{\cX_{k}, \ell_{k}\}$ is a tuple with $\cX_{k} \in \mR^{n   m}$  being the input data and $\ell_{k}$ quantifies the loss incurred by $\cX_{k}$ using the model $g$ for the task at $k$.   We denote a sequence of ${\vect{\theta}}_{k}$ as $\vect{u}_{k:K} = \{{\vect{\theta}}_{\tau} \in \Omega_{\theta}, k \leq \tau \leq K \},$ with  $\Omega_{\theta}$ being the compact~ (feasible ) set for the parameters. We denote the optimal value with a superscript ${ (* )};$ for instance, we use ${\vect{\theta}}_k^{ (* )}$ to denote the optimal value of ${\vect{\theta}}_{k}$ at task $k.$ In this paper we  use balance point, equilibrium point, and saddle  point to refer to the point of balance between generalization and forgetting. We  interchange between these terms whenever convenient for the discussion. We will use $\nabla_{ (j )} i$ to denote the gradient of $i$ with respect to $j$ and $\Delta i$ to denote the first difference in discrete time.
\section{Additional Results}
 We define the cost~ (combination of catastrophic cost and generalization cost ) at any instant $k$ as $J_{k} ({\vect{\theta} }_{k} ) = \gamma_{k} \ell_{k} + \sum_{\tau = 0}^{k-1} \gamma_{\tau} \ell_{\tau},$ where $\ell_{\tau}$ is computed on task $\cT_{\tau}$ with $\gamma_{\tau}$ describing the contribution of $\cT_{\tau}$ to this sum. We will show that for any fixed $k$, the catastrophic forgetting cost~$J_k ({\vect{\theta} }_{k} )$ is divergent in the limit $k \rightarrow \infty$ if equal contribution from each task is expected.
\begin{lemma}
 For any $k \in \mN,$  define $J_k ({\vect{\theta} }_{k} ) = \sum_{\tau = 0}^{k}  \gamma_{\tau} \ell_{\tau}.$ For all $\tau,$ assume  $\ell_{\tau}$ to be continuous with $L \geq \ell_{\tau} \geq \epsilon, \forall \tau, \epsilon >0$ and let $\gamma_{\tau} =1$. Then $J_k ({\vect{\theta} }_{k} )$ is divergent as $k \rightarrow \infty.$ 
\end{lemma}
\begin{proof}[Proof of Lemma 1]
	With $J_k ({\vect{\theta} }_{k} ) = \sum_{\tau = 0}^{k}  \gamma_{\tau} \ell_{\tau}$ we write $\lim_{k \rightarrow \infty} \sum_{\tau = 0}^{k}  \gamma_{\tau} \ell_{\tau} \geq \lim_{k \rightarrow \infty}  \sum_{\tau = 0}^{k}  \gamma_{\tau} \epsilon $, where $\gamma_{\tau}=1$ which provides $ \lim_{k \rightarrow \infty} \sum_{\tau = 0}^{k} \epsilon = \infty.$ Therefore, $J_k ({\vect{\theta} }_{k} )$ is divergent.
\end{proof}
 When $\ell_{\tau} \geq \epsilon$ with $\epsilon >0$ implies that each task incurs a nonzero cost. Furthermore, $\gamma_{\tau} =1,$ it implies that each task provides equal contribution to the catastrophic forgetting cost and  contributed nonzero value to $J_k ({\vect{\theta} }_{k} )$. The aforementioned lemma demonstrates that equivalent performance~ (no forgetting  on all tasks ) cannot be guaranteed for an infinite number of tasks when each task provides a nonzero cost to the sum~ (you have to learn for all the tasks ). However, if the task contributions are prioritized based on knowledge about the task distribution, the sum can be ensured to be convergent as shown in the next corollary.
 \begin{corr}
    	For any $k \in \mN,$ define  $J_k ({\vect{\theta} }_{k} ) = \sum_{\tau=0}^{k} \gamma_{\tau} \ell_{\tau}$ where $\ell_{\tau}$ is continuous and bounded such that $\epsilon \leq \ell_{\tau} \leq L, \forall \epsilon >0.$ Define $N= \frac{1}{k}$ and choose $\gamma_{N}$ such that $\gamma_{N} \rightarrow 0, N \rightarrow \infty$ and assume when there are infinite number of tasks, $\lim_{ N \rightarrow \infty} \sum_{N} \gamma_{N} \leq M.$ Under these assumptions, $J_k ({\vect{\theta} }_{k} )$ is convergent.
\end{corr}
\begin{proof}[Proof of Corollary 1]
 Since $\ell_{\tau} \leq L$, $J_k ({\vect{\theta} }_{k} ) = lim_{k \rightarrow \infty} \sum_{\tau=0}^{k}  \gamma_{\tau} \ell_{\tau} \leq lim_{k \rightarrow \infty} \sum_{\tau=0}^{k}  \gamma_{\tau} L \leq  L   lim_{k \rightarrow \infty} \sum_{\tau = 0}^{k} \gamma_{\tau}.$
 
 Since $\lim_{k \rightarrow \infty} \sum_{\tau = 0}^{k} \gamma_{\tau} = \lim_{ N \rightarrow \infty} \sum_{N} \gamma_{N}$ as $N =\frac{1}{k}$, therefore $\lim_{k \rightarrow \infty} \sum_{\tau = 0}^{k} \gamma_{\tau} \leq M$ and $J_k ({\vect{\theta} }_{k} )$ is upper bounded by $ L   M$. As a result, $J_k ({\vect{\theta} }_{k} )$ is convergent since $J_k ({\vect{\theta} }_{k} )$ is a monotone.
\end{proof}

To solve the problem at $k$, we seek ${\vect{\theta} }_{k}$ to minimize $J_{k} ({\vect{\theta} }_{k} )$. Similarly, to solve the problem in the complete interval $[0, K )$, we seek a ${\vect{\theta} }_{k}$ to minimize $J_{k} ({\vect{\theta}}_{k} )$ for each $k \in [0,K ).$ In other words we seek to obtain ${\vect{\theta} }_{k}$ for each task such that the cost $J_{k} ({\vect{\theta} }_{k} )$ is minimized. The optimization problem for the overall CL problem~ (overarching goal of CL ) is then provided as the minimization of the cumulative cost $V_{k} (\vect{u}_{k:K} ) = \sum_{\tau=k}^{K} \beta_{\tau} J_{\tau} ( {\vect{\theta}}_{\tau} )$ such that $V_k^{ (* )},$  is given as  \begin{equation} V_k^{ (* )} = min_{\vect{u}_{k:K} }  V_{k} (\vect{u}_{k:K} ), \label{op} \end{equation} with $0 \leq \beta_{\tau} \leq 1$ being the contribution of $J_{\tau}$ and $\vect{u}_{k:K}$ being a weight sequence of length $K-k.$  We will now derive the difference equation for our cost formulation.
\begin{prop}
For any $k \in [0, K ),$ define  $ V_k = \sum_{\tau=k}^{K} \beta_{\tau} J_{\tau} ( {\vect{\theta}}_{\tau} )$ with ${\vect{\theta}}_{\tau} \in \Omega.$  Define $\vect{u}_{k:K} = \{{\vect{\theta}}_{\tau} \in \Omega, k \leq \tau \leq K \},$, with  $\Omega$ being the compact~ (feasible ) set as a sequence of parameters with length $K-k$ and $V_k^{ (* )} = min_{\vect{u}_{k:K} }  \sum_{\tau=k}^{K} \beta_{\tau} J_{\tau} ( {\vect{\theta}}_{\tau} ).$ Then, the following is true
\begin{equation}
	\begin{aligned}
	      \Delta V^{ (* )}_{k}  =	- min_{{\vect{\theta}}_{k} \in \Omega}   \big[ \beta_k J_{k} ( {\vect{\theta}}_{k} )  +  \big ( \langle \nabla_{{\vect{\theta}}_{k}} V_{k}^{ (* )} , \Delta {\vect{\theta}}_{k} \rangle  +\langle \nabla_{\vect{x}_{k}} V_{k}^{ (* )}, \Delta \vect{x}_{k} \rangle \big )\big],&
	\end{aligned}
	\label{eq_opt}
\end{equation}
where $\Delta V^{ (* )}_{k}$ represents the first difference due to the introduction of a task, $\Delta {\vect{\theta}}_{k} $ due to parameters and $\nabla_{\vect{x}_{k}}$ due to the task data with $\beta_k  \in \mathbb{R} \cup [0,1], \forall k$ and $J_{k} ({\vect{\theta} }_{k} ) = \gamma_{k} \ell_{k} + \sum_{\tau = 0}^{k-1} \gamma_{\tau} \ell_{\tau}$.
\label{Der_HJB}
\end{prop}
\begin{proof}
Given $V_k^{ (* )} = min_{\vect{u}_{k:K} }  \sum_{\tau=k}^{K} \beta_{\tau} J_{\tau} ( {\vect{\theta}}_{\tau} ),$ we split the interval $[k,  K )$ as $[k, k+1 )$ and $[k+1,  K )$ to write $$V^{ (* )}_{k}= min_{{\vect{\theta}}_{\tau} \in \Omega}  \big[  \beta_k   J_{k} ( {\vect{\theta}}_{k} )\big]  +  min_{\vect{u}_{k+1:K} }\big[ \sum_{\tau=k+1}^{K}  \beta_{\tau} J_{\tau} ( {\vect{\theta}}_{\tau} ) \big].$$
$V_{k}= \sum_{\tau=k}^{K} \beta_{\tau} J_{\tau} ({\vect{\theta}}_{\tau} )$ provides $\sum_{\tau=k+1}^{K} \beta_{\tau} J_{\tau} ( {\vect{\theta}}_{\tau} )$ is $V_{k+1}$ therefore $min_{\vect{u}_{k+1:K} }\big[ \sum_{\tau=k+1}^{K}  \beta_{\tau} J_{\tau} ( {\vect{\theta}}_{\tau} ) \big]$ is $V^{ (* )}_{k+1}.$ We then achieve 
$$V^{ (* )}_{k}   =  min_{{\vect{\theta}}_{k} \in \Omega}   \big[  \beta_k  J_{k} ( {\vect{\theta}}_{k} ) + V^{ (* )}_{k+1} \big].$$ 

Since the minimization is with respect to $k$ now, the terms in $k+1$ can be pulled into of the bracket without any change to the minimization problem. We then approximate $V^{ (* )}_{k+1}$ using the information provided at $k.$ Since $V^{ (* )}_{k+1}$ is a function of $\vect{y}_{k},$ which is then a function of $ (k, \vect{x}_{k}, {\vect{\theta}}_{k} ),$ and all changes in  $\vect{y}_{k}$ can be summarized through $ (k, \vect{x}_{k}, {\vect{\theta}}_{k} ).$ Therefore,  a Taylor series of $V^{ (* )}_{k+1}$ around $ (k, \vect{x}_{k}, {\vect{\theta}}_{k} )$ provides
\begin{equation}
	\begin{aligned}
		&V^{ (* )}_{k+1} = V^{ (* )}_{k} + \langle \nabla_{{\vect{\theta}}_{k}} V_{k}^{ (* )} , \Delta {\vect{\theta}}_{k} \rangle &\\  &+ \langle \nabla_{\vect{x}_{k}} V_{k}^{ (* )}, \Delta \vect{x}_{k} \rangle + \langle \nabla_{k} (V^{ (* )}_{k} ), \Delta k \rangle + \cdots, &
	\end{aligned}
	\label{eq_a_1_5}
\end{equation}
where $\cdots$ summarizes all the higher order terms. As $k \in \mathbb{N}$ and $\langle \nabla_{k} (V^{ (* )}_{k} ), \Delta k \rangle$ represents the first difference in $V^{ (* )}_{k}$ hitherto denoted by $\Delta V^{ (* )}_{k}.$ We therefore achieve
\begin{equation}
	\begin{aligned}
	V^{ (* )}_{k+1} &= V^{ (* )}_{k} + \langle \nabla_{{\vect{\theta}}_{k}} V_{k}^{ (* )}, \Delta~\vect{\theta}_{k} \rangle \\  
    &+ \langle \nabla_{\vect{x}_{k}} V_{k}^{ (* )}, \Delta \vect{x}_{k} \rangle + \Delta V^{ (* )}_{k} + \cdots, 
	\end{aligned}
	\label{eq_a_1_5}
\end{equation}
Substitute into the original equation to get
\begin{equation}
	\begin{aligned}
		& V^{ (* )}_{k}   =  min_{{\vect{\theta}}_{k} \in \Omega}   \big[ \beta_k   J_{k} ( {\vect{\theta}}_{k} ) \big] +  \big (  V^{ (* )}_{k} + \langle \nabla_{{\vect{\theta}}_{k}} V_{k}^{ (* )} , \Delta {\vect{\theta}}_{k} \rangle  & \\ &+ \langle \nabla_{\vect{x}_{k}} V_{k}^{ (* )}, \Delta \vect{x}_{k} \rangle +\Delta V^{ (* )}_{k} \big ) + \cdots,& 
	\end{aligned}
	\label{eq_a_1_5}
\end{equation}
Cancel common terms and assume that the higher order terms~ ($\cdots$ ) are negligible to obtain 
\begin{equation}
	\begin{aligned}
 \Delta V^{ (* )}_{k} =	- min_{{\vect{\theta}}_{k} \in \Omega}   \big[ \beta_k J_{k} ( {\vect{\theta}}_{k} )  +  \langle \nabla_{{\vect{\theta}}_{k}} V_{k}^{ (* )} , \Delta {\vect{\theta}}_{k} \rangle  +\langle \nabla_{\vect{x}_{k}} V_{k}^{ (* )}, \Delta \vect{x}_{k} \rangle \big].&
	\end{aligned}
	\label{eq_opt}
\end{equation}
which is a difference equation in $V^{ (* )}_{k}.$
\end{proof}

Note that $V^{ (* )}_{k}$ is the minima for the overarching CL problem and $\Delta V^{ (* )}_{k}$ represents the change in $V^{ (* )}_{k}$ upon introduction of a task~ (we hitherto refer to this as perturbations ). Zero perturbations $ (\Delta V^{ (* )}_{k}=0 )$ implies that the introduction of a new task does not impact our current solution; that is, the optimal solution on all previous tasks is optimal on the new task as well. 

The solution of the CL problem can directly be obtained by solving Eq. \eqref{eq_opt} using all the available data. Thus, $ min_{{\vect{\theta}}_{k} \in \Omega}   \big[ H (\Delta \vect{x}_{k}, \vect{\theta}_{k} )  \big] \quad \text{yields } \Delta V^{ (* )}_{k} \approx 0$ for $\beta > 0,$ with $H (\Delta \vect{x}_{k}, \vect{\theta}_{k} ) =  \beta_k J_{k} ( {\vect{\theta}}_{k} ) +   \langle \nabla_{{\vect{\theta}}_{k}} V_{k}^{ (* )} , \Delta {\vect{\theta}}_{k} \rangle + \langle \nabla_{\vect{x}_{k}} V_{k}^{ (* )}, \Delta \vect{x}_{k} \rangle.$ Essentially, minimizing $H (\Delta \vect{x}_{k}, \vect{\theta}_{k} )$  would minimize the perturbations introduced by any new task. 

We simulate  worst-case discrepancy by iteratively updating $\Delta \vect{x}_{k}$ through gradient ascent, thus maximizing generalization. Next, we minimize forgetting under maximum generalization by iteratively updating $\vect{\theta}_{k}$ through gradient descent. To formalize our idea, let us indicate the iteration index at $k$ by $i$ and write $\Delta \vect{x}_{k}$ as $\Delta \vect{x}^{ (i )}_{k}$ and ${\vect{\theta}}_{k}$ as ${\vect{\theta}}^{ (i )}_{k}$ with $H (\Delta \vect{x}_{k}, \vect{\theta}_{k} )$ as $H (\Delta \vect{x}^{ (i )}_{k}, \vect{\theta}^{ (i )}_{k} )$~ (for simplicity of notation, we will denote $H (\Delta \vect{x}^{ (i )}_{k}, \vect{\theta}^{ (i )}_{k} )$ as $H$ whenever convenient ). Towards these updates, we will first get an upper bound on $H (\Delta \vect{x}^{ (i )}_{k}, \vect{\theta}^{ (i )}_{k} )$ and solve the upper bounding problem.

\begin{prop}
Let $k \in [0,  K )$ and define $H (\Delta \vect{x}^{ (i )}_{k}, \vect{\theta}^{ (i )}_{k} ) = \beta_k J_{k} ( {\vect{\theta}}^{ (i )}_{k} ) +   \langle \nabla_{{\vect{\theta}}_{k}} V_{k}^{ (* )} , \Delta {\vect{\theta}}^{ (i )}_{k} \rangle + \langle \nabla_{\vect{x}^{ (i )}_{k}} V_{k}^{ (* )}, \Delta \vect{x}^{ (i )}_{k} \rangle$  assume that  $ \nabla_{{\vect{\theta}}_{k}} V_{k}^{ (* )} \leq  \nabla_{{\vect{\theta}}_{k}} J_{k} ({\vect{\theta}}^{ (i )}_{k} ).$ Then the following approximation is true:
\begin{equation}
	\begin{aligned}
		& H (\Delta \vect{x}^{ (i )}_{k}, \vect{\theta}^{ (i )}_{k} ) \leq \beta_k J_{k} (\vect{\theta}^{ (i )}_{k} ) +  (J_{k} (\vect{\theta}^{ (i+\zeta )}_{k} ) - J_{k} (\vect{\theta}^{ (i )}_{k} ) ) +  ( J_{k+\zeta} (\vect{\theta}^{ (i )}_{k} ) - J_{k} (\vect{\theta}^{ (i )}_{k} )  ),&
	\end{aligned}
	\label{eq_a_1_6}
\end{equation}
where $\beta_k  \in \mathbb{R} \cup [0,1], \forall k$ and $\zeta \in \mathbb{N}$ and $J_{k+\zeta}$ indicates $\zeta$ updates on $\Delta \vect{x}^{ (i )}_{k}$ and $\vect{\theta}^{ (i+\zeta )}_{k}$ indicates $\zeta$ updates on ${\vect{\theta}}^{ (i )}_{k}.$
\label{prop2}
\end{prop}
\begin{proof}
Consider $	H (\Delta \vect{x}^{ (i )}_{k}, \vect{\theta}^{ (i )}_{k} ) =  \beta_k J_{k} ( {\vect{\theta}}^{ (i )}_{k} ) +   \langle \nabla_{{\vect{\theta}}_{k}} V_{k}^{ (* )} , \Delta {\vect{\theta}}^{ (i )}_{k} \rangle + \langle \nabla_{\vect{x}^{ (i )}_{k}} V_{k}^{ (* )}, \Delta \vect{x}^{ (i )}_{k} \rangle.$  Assuming $\nabla_{{\vect{\theta}}_{k}} V_{k}^{ (* )} \leq \nabla_{{\vect{\theta}}_{k}} J_{k} ({\vect{\theta}}^{ (i )}_{k} )$ we may write through finite difference approximation as
\begin{equation}
	\begin{aligned}
    \langle \nabla_{{\vect{\theta}}^{ (i )}_{k}} V_{k}^{ (* )} , \Delta {\vect{\theta}}^{ (i )}_{k} \rangle &\leq& \langle \nabla_{{\vect{\theta}}^{ (i )}_{k}} J_{k} ({\vect{\theta}}^{ (i )}_{k} ) , \Delta {\vect{\theta}}^{ (i )}_{k} \rangle,\\
     &\leq&   (J_{k} (\vect{\theta}^{ (i+\zeta )}_{k} ) -    J_{k} (\vect{\theta}^{ (i )}_{k} ) )
	\end{aligned}
	\label{eq_a_1_6}
\end{equation}
Similarly, we may write
\begin{equation}
	\begin{aligned}
  \langle \nabla_{\vect{x}_{k}} V_{k}^{ (* )}, \Delta \vect{x}_{k} \rangle &\leq&
  \langle \nabla_{\vect{x}_{k}} J_{k} ({\vect{\theta}}^{ (i )}_{k} ), \Delta \vect{x}_{k} \rangle,\\
   &\leq&  ( J_{k+\zeta} (\vect{\theta}^{ (i )}_{k} ) - J_{k} (\vect{\theta}^{ (i )}_{k} )  ).
	\end{aligned}
	\label{eq_a_1_6}
\end{equation}
Upon substitution, we have our result:
\begin{equation}
    \begin{aligned}
    H (\Delta \vect{x}^{ (i )}_{k}, \vect{\theta}^{ (i )}_{k} ) \leq \beta_k J_{k} (\vect{\theta}^{ (i )}_{k} ) + (J_{k} (\vect{\theta}^{ (i+\zeta )}_{k} ) - J_{k} (\vect{\theta}^{ (i )}_{k} ) ) +  (J_{k+\zeta} (\vect{\theta}^{ (i )}_{k} ) - J_{k} (\vect{\theta}^{ (i )}_{k} ) ).
    \end{aligned}
    \label{eq_a_1_6}
\end{equation}
\end{proof}

Our cost to be analyzed will be given as 
\begin{equation}
	\begin{aligned}
	H (\Delta \vect{x}^{ (i )}_{k}, \vect{\theta}^{ (i )}_{k} ) = \beta_k J_{k} ( {\vect{\theta}}^{ (i )}_{k} ) +   \langle \nabla_{{\vect{\theta}}^{ (i )}_{k}} V_{k}^{ (* )} , \Delta {\vect{\theta}}^{ (i )}_{k} \rangle + \langle \nabla_{\vect{x}^{ (i )}_{k}} V_{k}^{ (* )}, \Delta \vect{x}^{ (i )}_{k} \rangle.
	\end{aligned}
	\label{eq_approx}
\end{equation}
and use this definition of $H (\Delta \vect{x}^{ (i )}_{k}, \vect{\theta}^{ (i )}_{k} )$ from here on.
\section{Main results}
\begin{figure}
  \includegraphics[width = \columnwidth]{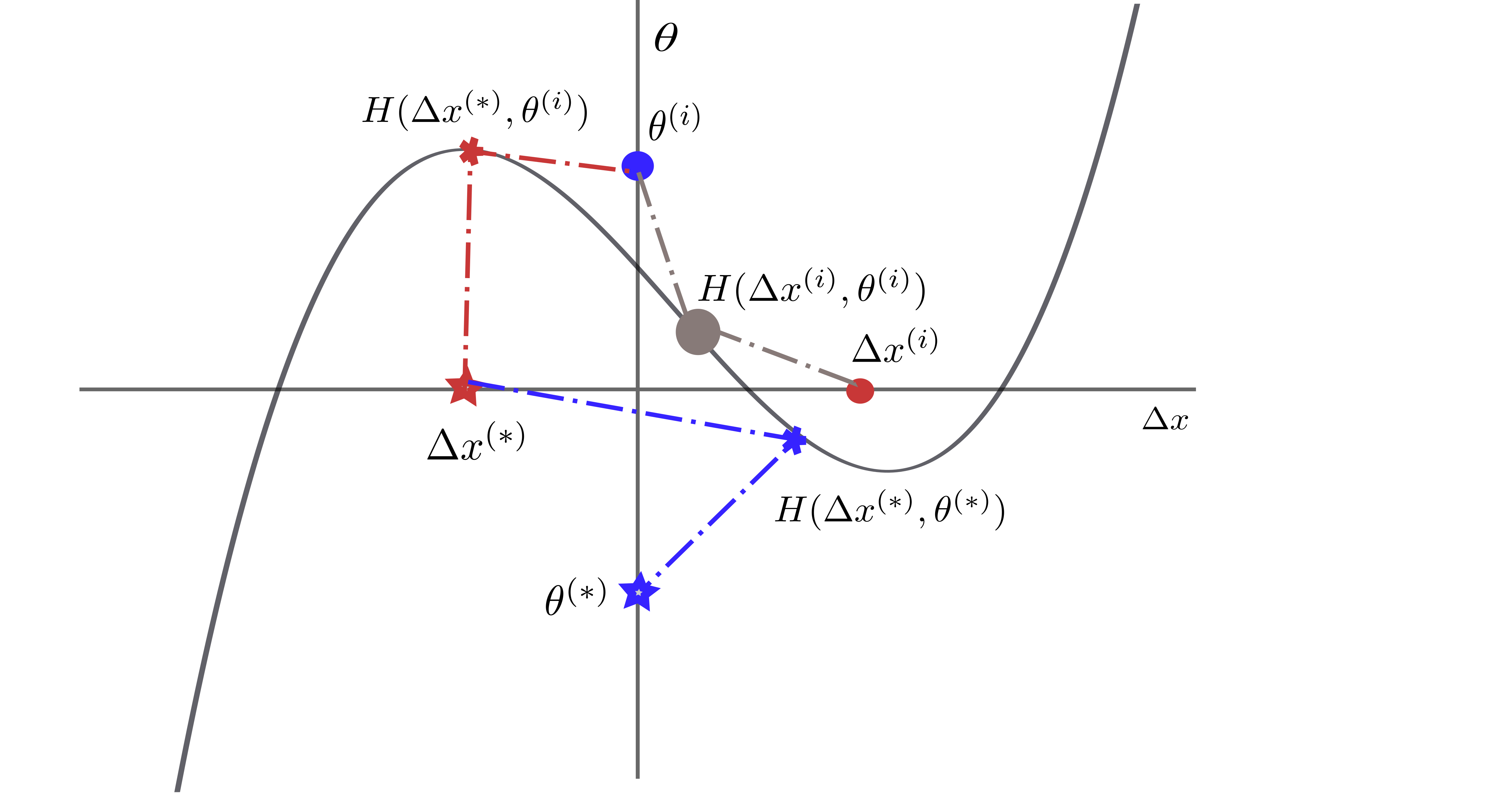}
    \caption{Illustration of proofs. $\Delta \vect{x}$~ (player~1 ) is the horizontal axis and the vertical axis indicates $\vect{\theta}$~ (player~2 ) where the curve indicates H. If we start from red circle for player~1~ (player~2 is fixed at the  blue circle ) H is increasing~ (goes from a grey circle to a red asterisk ) with player 1 reaching the red star. Next,  start from the blue circle~ ($\vect{\theta}$ is at the red star ), the cost decreases.}
    \label{fig:proof}
\end{figure} 
We will define two compact sets $\Omega_{\theta}, \Omega_{x}$ and seek to show existence and stability of a saddle point~$ (\Delta \vect{x}^{ (i )}_{k} , \vect{\theta}_k^{ (i )} )$ for a fixed $k.$ To illustrate the theory, we refer to Fig.~\ref{fig:proof}, for each $k,$ the initial values for the two players are characterized by the pair  $\{\vect{\theta}^{ (i )}_{k} \text{ (blue circle )}, \Delta \vect{x}^{ (i )}_{k} \text{ (red circle )} \},$ and $H ( \Delta \vect{x}^{ (i )}_{k}, \vect{\theta}^{ (i )}_{k} )$ is indicated by the grey circle on the cost curve~ (the dark blue curve ).  Our proofing strategy is as follows. 

First, we fix $\vect{\theta}_k^{ (. )} \in \Omega_{\theta}$ and construct $\mathcal{M}_k= \{\vect{\theta}^{ (. )}_{k}, \Omega_{x} \},$ to prove that H is maximizing with respect to $\Delta \vect{x}^{ (i )}_{k}.$ 
\begin{lemma}
For each $k \in [0, K ),$  fix $\vect{\theta}_k^{ (. )} \in \Omega_{\theta}$ and construct $\mathcal{M}_k = \{\Omega_{x}, \vect{\theta}^{ (. )}_{k}\}$  with $\Omega_{\theta}, \Omega_{x}$ being the sets of all feasible $\vect{\theta}^{ (. )}_{k}$ and $\vect{x}^{ (i )}_{k}$ respectively. Define $H (\Delta \vect{x}^{ (i )}_{k}, \vect{\theta}^{ (. )}_{k}  )$ as in Eq. \eqref{eq_approx} for $ (\Delta \vect{x}^{ (i )}_{k}, \vect{\theta}^{ (. )}_{k}  ) \in \mathcal{M}_k $ and consider $$ \Delta \vect{x}^{ (i+1 )}_{k} - \Delta \vect{x}^{ (i )}_{k} = \alpha_{k}^{ (i )}  \nabla_{\Delta \vect{x}^{ (i )}_{k}} H (\Delta \vect{x}^{ (i )}_{k}, \vect{\theta}^{ (. )}_{k} )  ) / \| \nabla_{\Delta \vect{x}^{ (i )}_{k}} H (\Delta \vect{x}^{ (i )}_{k}, \vect{\theta}^{ (. )}_{k} ) \|^2.$$ Consider the assumptions $\nabla_{ \vect{x}^{ (i )}_{k}} V_{k}^{ (* )} \leq \nabla_{ \vect{x}^{ (i )}_{k}} J_{k}$ and   $\langle \nabla_{\vect{x}^{ (i )}_{k}} J_k, \nabla_{\vect{x}^{ (i )}_{k}} J_k \rangle>0$, and let $\alpha_{k}^{ (i )} \rightarrow 0, i \rightarrow \infty.$ It follows that $H (\Delta \vect{x}^{ (i )}_{k}, \vect{\theta}^{ (. )}_{k} )$ converges asymptotically to a maximizer.
\label{lem:lem_max}
\end{lemma}
\begin{proof}
Fix $\vect{\theta}^{ (. )}_{k} \in \Omega_{\theta}$ and construct $\mathcal{M}_k$ such that $\mathcal{M}_k = \{\vect{\theta}^{ (. )}_{k}, \Omega_{x} \}$ which we call a neighborhood. Therefore, for $ (\Delta \vect{x}^{ (i+1 )}_{k}, \vect{\theta}^{ (. )}_{k} ),  (\Delta \vect{x}^{ (i )}_{k}, \vect{\theta}^{ (. )}_{k} ) \in \mathcal{M}_k$ we may write a first-order Taylor series expansion of  $H (\Delta \vect{x}^{ (i+1 )}_{k}, \vect{\theta}^{ (. )}_{k} )$ around $H (\Delta \vect{x}^{ (i )}_{k}, \vect{\theta}^{ (. )}_{k} )$ as
\begin{equation}
\begin{aligned}
	H (\Delta \vect{x}^{ (i+1 )}_{k}, \vect{\theta}^{ (. )}_{k} ) = H (\Delta \vect{x}^{ (i )}_{k}, \vect{\theta}^{ (. )}_{k} ) +  \langle \nabla_{\Delta \vect{x}^{ (i )}_{k}} H (\Delta \vect{x}^{ (i )}_{k}, \vect{\theta}^{ (. )}_{k} ), \Delta \vect{x}^{ (i+1 )}_{k} - \Delta \vect{x}^{ (i )}_{k}  \rangle .
\end{aligned}
\end{equation}
We substitute the update as $\alpha_{k}^{ (i )}  \frac{\nabla_{\Delta \vect{x}^{ (i )}_{k}} H (\Delta \vect{x}^{ (i )}_{k}, \vect{\theta}^{ (. )}_{k} )}{\| \nabla_{\Delta \vect{x}^{ (i )}_{k}} H (\Delta \vect{x}^{ (i )}_{k}, \vect{\theta}^{ (. )}_{k} ) \|^2}$ to get 
 \begin{equation}
	\begin{aligned}
    H (\Delta \vect{x}^{ (i+1 )}_{k}, \vect{\theta}^{ (. )}_{k} ) - H (\Delta \vect{x}^{ (i )}_{k}, \vect{\theta}^{ (. )}_{k} )=  \langle \nabla_{\Delta \vect{x}^{ (i )}_{k}} H (\Delta \vect{x}^{ (i )}_{k}, \vect{\theta}^{ (. )}_{k} ), \alpha_{k}^{ (i )}  \frac{\nabla_{\Delta \vect{x}^{ (i )}_{k}} H (\Delta \vect{x}^{ (i )}_{k}, \vect{\theta}^{ (. )}_{k} )}{\| \nabla_{\Delta \vect{x}^{ (i )}_{k}} H (\Delta \vect{x}^{ (i )}_{k}, \vect{\theta}^{ (. )}_{k} )\|^2} \rangle.
	\end{aligned}
\end{equation}
The derivative $\nabla_{\Delta \vect{x}^{ (i )}_{k}} H (\Delta \vect{x}^{ (i )}_{k}, \vect{\theta}^{ (. )}_{k} )$ can be written as 
 \begin{equation}
	\begin{aligned}
    & \nabla_{\Delta \vect{x}^{ (i )}_{k}} H (\Delta \vect{x}^{ (i )}_{k}, \vect{\theta}^{ (. )}_{k} )  ) \leq \nabla_{\Delta \vect{x}^{ (i )}_{k}} \Big[ \beta_k J_{k} ( {\vect{\theta}}^{ (. )}_{k} ) +   \langle \nabla_{{\vect{\theta}}^{ (. )}_{k}} V_{k}^{ (* )} , \Delta {\vect{\theta}}^{ (. )}_{k} \rangle & \\  &+ \langle \nabla_{\vect{x}^{ (i )}_{k}} V_{k}^{ (* )}, \Delta \vect{x}^{ (i )}_{k} \rangle  \Big] = \nabla_{ \vect{x}^{ (i )}_{k}} V_{k}^{ (* )} \leq \nabla_{ \vect{x}^{ (i )}_{k}} J_{k}.&
	\end{aligned}
\end{equation}
Substitution reveals
 \begin{equation}
	\begin{aligned}
    	H (\Delta \vect{x}^{ (i+1 )}_{k}, \vect{\theta}^{ (. )}_{k} ) - H (\Delta \vect{x}^{ (i )}_{k}, \vect{\theta}^{ (. )}_{k} ) = \alpha_{k}^{ (i )}  \frac{ \langle \nabla_{\vect{x}^{ (i )}_{k}}  J_{k},  \nabla_{\vect{x}^{ (i )}_{k}}  J_{k} \rangle
    	  }{\|\nabla_{\vect{x}^{ (i )}_{k}} J_{k} \|^2}  
	\end{aligned}
\end{equation}
for $\alpha_{k}^{ (i )} >0$; and under the assumption that $ \langle \nabla_{\vect{x}^{ (i )}_{k}}  J_{k},  \nabla_{ \vect{x}^{ (i )}_{k}} J_{k} \rangle >0$ we obtain
\begin{equation}
	\begin{aligned}
    	 H (\Delta \vect{x}^{ (i+1 )}_{k}, \vect{\theta}^{ (. )}_{k} ) - H (\Delta \vect{x}^{ (i )}_{k}, \vect{\theta}^{ (. )}_{k} ) =   \alpha_{k}^{ (i )}  \frac{ \langle \nabla_{\vect{x}^{ (i )}_{k}} J_{k},  \nabla_{\vect{x}^{ (i )}_{k}} J_{k}
    	 \rangle }{\| \nabla_{\vect{x}^{ (i )}_{k}} J_{k} \|^2 } \geq 0.
	\end{aligned}
\end{equation}
Let $ B_{x} = \alpha_{k}^{ (i )}  \frac{ \langle \nabla_{\vect{x}_{k}} J_{k},  \nabla_{ \vect{x}_{k}} J_{k} \rangle }{\| \nabla_{ \vect{x}^{ (i )}_{k}} J_{k} \|^2 } \leq \alpha_{k}^{ (i )}$ and therefore $0 \leq H (\Delta \vect{x}^{ (i+1 )}_{k}, \vect{\theta}^{ (. )}_{k} ) - H (\Delta \vect{x}^{ (i )}_{k}, \vect{\theta}^{ (. )}_{k} )\leq \alpha^{ (i )}_{k}.$ We therefore have $H (\Delta \vect{x}^{ (i+1 )}_{k}, \vect{\theta}^{ (. )}_{k} ) - H (\Delta \vect{x}^{ (i )}_{k}, \vect{\theta}^{ (. )}_{k} ) \geq 0$ and $H (\Delta \vect{x}^{ (i )}_{k}, \vect{\theta}^{ (. )}_{k} )$ is maximizing with respect to $\Delta \vect{x}^{ (i )}_{k}.$ Furthermore, under the assumption that $\alpha^{ (i )}_{k} \rightarrow 0, k \rightarrow \infty$,  we have $H (\Delta \vect{x}^{ (i+1 )}_{k}, \vect{\theta}^{ (. )}_{k} ) - H (\Delta \vect{x}^{ (i )}_{k}, \vect{\theta}^{ (. )}_{k} ) \rightarrow 0$ asymptotically and we have our result.
\end{proof}

Similarly, we fix $\Delta \vect{x}^{ (. )}_{k} \in \Omega_x$ and construct $\mathcal{N}_k= \{ \Omega_{\theta}, \Delta \vect{x}^{ (. )}_{k}  \},$ to prove that $H$ is minimizing with respect to $\vect{\theta}^{ (i )}_{k}.$ 
\begin{lemma}
For each $k \in [0, K ),$ fix $\Delta \vect{x}^{ (. )}_{k} \in \Omega_{x}$ and construct $\mathcal{N}_k= \{\Delta \vect{x}^{ (. )}_{k}, \Omega_{\theta} \}.$ Then for any $ (\Delta \vect{x}^{ (i )}_{k}, \vect{\theta}^{ (. )}_{k}  ) \in \mathcal{N}_k$ define $H (\Delta \vect{x}^{ (i )}_{k}, \vect{\theta}^{ (. )}_{k}  )$ as in Eq. \eqref{eq_approx} with Proposition.~\ref{prop2} being true and let $\vect{\theta}^{ (i+1 )}_{k} - \vect{\theta}^{ (i )}_{k} = -\alpha_{k}^{ (i )}  \nabla_{{\vect{\theta}}^{ (i )}_{k}} H (\Delta \vect{x}^{ (. )}_{k}, \vect{\theta}^{ (i )}_{k} ) )$. Assume that $\| \nabla_{\vect{\theta}^{ (i )}_{k}} J_{k} ( {\vect{\theta}}^{ (i+\zeta )}_{k} )\| \leq L_1$ and $\| \nabla_{\vect{\theta}^{ (i )}_{k}} J_{k+\zeta} ( {\vect{\theta}}^{ (i )}_{k} )\| \leq L_2$ and  let $\alpha_{k}^{ (i )} \rightarrow 0, i \rightarrow \infty.$ Then $\vect{\theta}^{ (i )}_{k}$ converges to a local minimizer.
\label{lem:lem_min}
\end{lemma}
\begin{proof}
First, we fix $\Delta \vect{x}^{ (. )}_k \in \Omega_{x}$ and construct $\mathcal{N}_k= \{ \Omega_{\theta} ,  \Delta \vect{x}^{ (. )}_{k} \}.$ For any $ (\Delta \vect{x}^{ (. )}_k, \vect{\theta}^{ (i )}_{k} ),  (\Delta \vect{x}^{ (. )}_k, \vect{\theta}^{ (i+1 )}_{k} ) \in\mathcal{N}_k$ we write  a first-order Taylor series expansion of  $H (\Delta \vect{x}^{ (. )}_k, \vect{\theta}^{ (i+1 )}_{k} )$ around $H (\Delta \vect{x}^{ (. )}_k, \vect{\theta}^{ (i )}_{k} )$ to write
 \begin{equation}
	\begin{aligned}
    	H (\Delta \vect{x}^{ (. )}_k, \vect{\theta}^{ (i+1 )}_{k} ) = H (\Delta \vect{x}^{ (. )}_k, \vect{\theta}^{ (i )}_{k} ) +  \langle \nabla_{\vect{\theta}^{ (i )}_{k}} H (\Delta \vect{x}^{ (. )}_k, \vect{\theta}^{ (i )}_{k} ),  \vect{\theta}^{ (i+1 )}_{k} - \vect{\theta}^{ (i )}_{k} \rangle . 
	\end{aligned}
\end{equation}
We then substitute 
$\vect{\theta}^{ (i+1 )}_{k} - \vect{\theta}^{ (i )}_{k} = -\alpha_{k}^{ (i )}  \nabla_{{\vect{\theta}}^{ (i )}_{k}} H (\Delta \vect{x}^{ (. )}_{k}, \vect{\theta}^{ (i )}_{k} ) $ to get
\begin{equation}
	\begin{aligned}
    	&H (\Delta \vect{x}^{ (. )}_k, \vect{\theta}^{ (i+1 )}_{k} ) - H (\Delta \vect{x}^{ (. )}_k, \vect{\theta}^{ (i )}_{k} ) = -\alpha_{k}^{ (i )}\langle \nabla_{\vect{\theta}^{ (i )}_{k}} H (\Delta \vect{x}^{ (. )}_k, \vect{\theta}^{ (i )}_{k} ), \nabla_{{\vect{\theta}}^{ (i )}_{k}} H (\Delta \vect{x}^{ (. )}_{k}, \vect{\theta}^{ (i )}_{k} )\rangle.&
	\end{aligned}
\end{equation}
Following Proposition 2, the derivative $\nabla_{{\vect{\theta}}^{ (i )}_{k}} H (\Delta \vect{x}^{ (. )}_{k}, \vect{\theta}^{ (i )}_{k} )$ can be written as 
\begin{equation}
	\begin{aligned}
    	\nabla_{{\vect{\theta}}^{ (i )}_{k}} H (\Delta \vect{x}^{ (. )}_{k}, \vect{\theta}^{ (i )}_{k} )\leq  \nabla_{\vect{\theta}_{k}} [\beta_k J_{k} ( {\vect{\theta}}^{ (i )}_{k} )  +    (J_{k} (\vect{\theta}^{ (i+\zeta )}_{k} ) - J_{k} (\vect{\theta}^{ (i )}_{k} ) ) +     (J_{k+\zeta} (\vect{\theta}^{ (i )}_{k} ) - J_{k} (\vect{\theta}^{ (i )}_{k} ) )]
	\end{aligned}
\end{equation}
Simplification reveals
 \begin{equation}
	\begin{aligned}
    \nabla_{{\vect{\theta}}^{ (i )}_{k}}  H (\Delta \vect{x}^{ (. )}_k, \vect{\theta}^{ (i )}_{k} ) )\leq 	\nabla_{{\vect{\theta}}^{ (i )}_{k}}   (\beta_k-2 ) J_{k} ( {\vect{\theta}}^{ (i )}_{k} )  +   \nabla_{{\vect{\theta}}^{ (i )}_{k}}  J_{k} (\vect{\theta}^{ (i+\zeta )}_{k} )  + \nabla_{\vect{\theta}_{k}} J_{k+\zeta} (\vect{\theta}^{ (i )}_{k} ).
	\end{aligned}
\end{equation}
Substitution therefore provides
\begin{align}
    	 &H (\Delta \vect{x}^{ (. )}_k, \vect{\theta}^{ (i+1 )}_{k} ) - H (\Delta \vect{x}^{ (. )}_k, \vect{\theta}^{ (i )}_{k} ) \nonumber\\ &\leq   
    	- \alpha_{k}^{ (i )}  \langle \nabla_{{\vect{\theta}}^{ (i )}_{k}}   (\beta_k-2 ) J_{k} ( {\vect{\theta}}^{ (i )}_{k} ) +   \nabla_{{\vect{\theta}}^{ (i )}_{k}}  J_{k} (\vect{\theta}^{ (i+\zeta )}_{k} )  + \nabla_{\vect{\theta}^{ (i )}_{k}} J_{k+\zeta} (\vect{\theta}^{ (i )}_{k} ), \nonumber
    	\\ & \nabla_{{\vect{\theta}}^{ (i )}_{k}}   (\beta_k-2 ) J_{k} ( {\vect{\theta}}^{ (i )}_{k} ) +   \nabla_{{\vect{\theta}}^{ (i )}_{k}}  J_{k} (\vect{\theta}^{ (i+\zeta )}_{k} )  + \nabla_{\vect{\theta}_{k}} J_{k+\zeta} (\vect{\theta}^{ (i )}_{k} ) \rangle .
\end{align}

Opening the square with Cauchy's inequality provides
 \begin{equation}
	\begin{aligned}
    	H (\Delta \vect{x}^{ (. )}_k, \vect{\theta}^{ (i+1 )}_{k} ) - H (\Delta \vect{x}^{ (. )}_k,     \vect{\theta}^{ (i )}_{k} ) 
    	&\leq& -    \alpha_{k}^{ (i )}  \Big[	\|\nabla_{{\vect{\theta}}^{ (i )}_{k}}   (\beta_k-2 ) J_{k} ( {\vect{\theta}}_{k} ) \|^2 \\ 
    	&+& \| \nabla_{{\vect{\theta}}^{ (i )}_{k}}  J_{k} (\vect{\theta}^{ (i+\zeta )}_{k} )\|^2
    	+ \| \nabla_{\vect{\theta}_{k}} J_{k+\zeta} (\vect{\theta}^{ (i )}_{k} )\|^2   
    	\\ 
    	&+& 2\|\nabla_{{\vect{\theta}}^{ (i )}_{k}}   (\beta_k-2 ) J_{k} ( {\vect{\theta}}_{k} ) \|  \| \nabla_{{\vect{\theta}}^{ (i )}_{k}}  J_{k} (\vect{\theta}^{ (i+\zeta )}_{k} )\|  \\ 
    	&+& 2\| \nabla_{{\vect{\theta}}^{ (i )}_{k}}  J_{k} (\vect{\theta}^{ (i+\zeta )}_{k} )\|\| \nabla_{\vect{\theta}_{k}} J_{k+\zeta} (\vect{\theta}^{ (i )}_{k} )\| 
      \\ 
    	&+& 2\|\nabla_{{\vect{\theta}}^{ (i )}_{k}}   (\beta_k-2 ) J_{k} ( {\vect{\theta}}_{k} ) \| \| \nabla_{\vect{\theta}_{k}} J_{k+\zeta} (\vect{\theta}^{ (i )}_{k} )\| \Big]. 
	\end{aligned}
\end{equation}
We simplify with Young's inequality to achieve
 \begin{equation}
	\begin{aligned}
    	 H (\Delta \vect{x}^{ (. )}_k, \vect{\theta}^{ (i+1 )}_{k} ) - H (\Delta \vect{x}^{ (. )}_k,     \vect{\theta}^{ (i )}_{k} ) 
    	&\leq& -    \alpha_{k}^{ (i )}  \Big[	\|\nabla_{{\vect{\theta}}^{ (i )}_{k}}   (\beta_k-2 )J_{k} ( {\vect{\theta}}_{k} ) \|^2 \\ 
    	&+& \| \nabla_{{\vect{\theta}}^{ (i )}_{k}}  J_{k} (\vect{\theta}^{ (i+\zeta )}_{k} )\|^2
    	\\ &+& \| \nabla_{\vect{\theta}_{k}} J_{k+\zeta} (\vect{\theta}^{ (i )}_{k} )\|^2   
    	\\ &+& \|\nabla_{{\vect{\theta}}^{ (i )}_{k}}   (\beta_k-2 ) J_{k} ( {\vect{\theta}}_{k} ) \|^2 \\ &+& \| \nabla_{{\vect{\theta}}^{ (i )}_{k}}  J_{k} (\vect{\theta}^{ (i+\zeta )}_{k} )\|^2  \\ 
    	&+& \| \nabla_{{\vect{\theta}}^{ (i )}_{k}}  J_{k} (\vect{\theta}^{ (i+\zeta )}_{k} )\|^2
    	\\ &+& \| \nabla_{\vect{\theta}_{k}} J_{k+\zeta} (\vect{\theta}^{ (i )}_{k} )\|^2 
        \\ &+& \|\nabla_{{\vect{\theta}}^{ (i )}_{k}}   (\beta_k-2 ) J_{k} ( {\vect{\theta}}_{k} ) \|^2 
        \\ &+& \| \nabla_{\vect{\theta}_{k}} J_{k+\zeta} (\vect{\theta}^{ (i )}_{k} )\|^2 \Big]. 
	\end{aligned}
\end{equation}
Further simplification results in 
\begin{equation}
	\begin{aligned}
    	 H (\Delta \vect{x}^{ (. )}_k, \vect{\theta}^{ (i+1 )}_{k} ) - H (\Delta \vect{x}^{ (. )}_k,     \vect{\theta}^{ (i )}_{k} ) 
    	&\leq& -    \alpha_{k}^{ (i )}  \Big[ 3	\|\nabla_{{\vect{\theta}}^{ (i )}_{k}}   (\beta_k-2 ) J_{k} ( {\vect{\theta}}_{k} ) \|^2 \\ 
    	&+& 3 \| \nabla_{{\vect{\theta}}^{ (i )}_{k}}  J_{k} (\vect{\theta}^{ (i+\zeta )}_{k} )\|^2
    	+ 3 \| \nabla_{\vect{\theta}_{k}} J_{k+\zeta} (\vect{\theta}^{ (i )}_{k} )\|^2   \Big]. 
	\end{aligned}
\end{equation}
With the assumption that $\| \nabla_{{\vect{\theta}}^{ (i )}_{k}}  J_{k} (\vect{\theta}^{ (i+\zeta )}_{k} )\| \leq L_1$ and $\| \nabla_{\vect{\theta}_{k}} J_{k+\zeta} (\vect{\theta}^{ (i )}_{k} )\|\leq L_2,$, we may write 
\begin{equation}
    \begin{aligned}
    	H (\Delta \vect{x}^{ (. )}_k, \vect{\theta}^{ (i+1 )}_{k} ) - H (\Delta \vect{x}^{ (. )}_k,     \vect{\theta}^{ (i )}_{k} )  &\leq&  -\alpha_{k}^{ (i )}  B_{\theta},
	\end{aligned}
\end{equation}
where $B_{\theta} = \Bigg[   ( (\sqrt{3}\beta_k-2\sqrt{3} )^2 + 3 )L_1^2 + 3L_2^2\Bigg].$ Assuming that $\alpha_{k}^{ (i )}$ is chosen such that $\alpha^{ (i )}_{k} \rightarrow 0$, we obtain $H (\Delta \vect{x}^{ (. )}_k, \vect{\theta}^{ (i+1 )}_{k} ) - H (\Delta \vect{x}^{ (. )}_k,     \vect{\theta}^{ (i )}_{k} ) \rightarrow 0$ as $i \rightarrow \infty$ and $H (\Delta \vect{x}^{ (. )}_k, \vect{\theta}^{ (i+1 )}_{k} ) - H (\Delta \vect{x}^{ (. )}_k,     \vect{\theta}^{ (i )}_{k} ) < 0.$ Therefore H converges to a local minimizer.
\end{proof}

From here on, we will define our cost function as H whereever convinient for simplicity of notations. Since, for any k, there exists a local maximizer $\Delta \vect{x}_k^{ (* )} \in \Omega_x,$ we may define $\mathcal{N}^{ (* )}_k = \{\Delta \vect{x}_k^{ (* )}, \Omega_\theta\}$ where the set $\Omega_{x}$ is comprised of a local maximizer $\Delta \vect{x}_k^{ (* )}$ 

\begin{lemma}
For any $k \in [0, K )$, let $\vect{\theta}^{ (* )}_{k} \in \Omega_\theta,$  be the minimizer of $H$ according to Lemma \ref{lem:lem_min} and define $\mathcal{M}^{ (* )}_k = \{\Omega_{x}, \vect{\theta}^{ (* )}_{k}\}.$ Then for $ (\Delta \vect{x}_k^{ (* )}, \vect{\theta}^{ (* )}_k ),  (\Delta \vect{x}_k^{ (i )}, \vect{\theta}^{ (* )}_k ) \in \mathcal{M}^{ (* )}_k,$ $H (\Delta \vect{x}_k^{ (* )}, \vect{\theta}^{ (* )}_k ) \geq H (\Delta \vect{x}_k^{ (i )}, \vect{\theta}^{ (* )}_k ),$ where $\Delta \vect{x}_k^{ (* )}$ is a maximizer for H according to Lemma. \ref{lem:lem_max}.
\label{lem:lem_max_opt}
\end{lemma}
\begin{proof}
By Lemma \ref{lem:lem_min}, for each $k \in [0, K ),$ there exists a minimizer $\vect{\theta}^{ (* )}_{k} \in \Omega_\theta$ such that $\mathcal{M}^{ (* )}_k = \{ \Omega_{x}, \vect{\theta}^{ (* )}_{k} \}.$ Therefore by Lemma \ref{lem:lem_max}, $H (\Delta \vect{x}^{ (i+1 )}_{k}, \vect{\theta}^{ (* )}_{k} ) - H (\Delta \vect{x}^{ (i )}_{k}, \vect{\theta}^{ (* )}_{k} ) \geq 0$ for  $ (\Delta \vect{x}_k^{ (i+1 )}, \vect{\theta}^{ (* )}_k ),  (\Delta \vect{x}_k^{ (i )}, \vect{\theta}^{ (* )}_k ) \in \mathcal{M}^{ (* )}_k.$  Let $\Delta \vect{x}_k^{ (* )} \in \Omega_x $ be the converging point according to Lemma \ref{lem:lem_max}. Then, for  $ (\Delta \vect{x}_k^{ (* )}, \vect{\theta}^{ (* )}_k ),  (\Delta \vect{x}_k^{ (i )}, \vect{\theta}^{ (* )}_k ) \in \mathcal{M}^{ (* )}_k$ a $H (\Delta \vect{x}^{ (* )}_{k}, \vect{\theta}^{ (* )}_{k} ) - H (\Delta \vect{x}^{ (i )}_{k}, \vect{\theta}^{ (* )}_{k} ) \geq 0.$ by Lemma~\ref{lem:lem_max} which provides the result.
\end{proof}


\begin{lemma}
For any $k \in [0, K )$, let $\Delta \vect{x}_k^{ (* )} \in \Omega_x,$  be the maximizer of $H$ according to Lemma \ref{lem:lem_max} and define $\mathcal{N}^{ (* )}_k = \{\Delta \vect{x}^{ (* )}_{k}, \Omega_{\theta}\}.$ Then for $ (\Delta \vect{x}_k^{ (* )}, \vect{\theta}^{ (* )}_k ),  (\Delta \vect{x}_k^{ (* )}, \vect{\theta}^{ (i )}_k ) \in \mathcal{N}^{ (* )}_k,$  $H (\Delta \vect{x}_k^{ (* )}, \vect{\theta}^{ (* )}_k ) \leq H (\Delta \vect{x}_k^{ (* )}, \vect{\theta}^{ (i )}_k ),$ where $\vect{\theta}^{ (* )}_{k}$ is a minimizer for H according to Lemma. \ref{lem:lem_max}.
\label{lem:lem_min_opt}
\end{lemma}
\begin{proof}
By Lemma \ref{lem:lem_max}, for each $k \in [0, K ),$ there exists a maximizer $\Delta \vect{x}_k^{ (* )} \in \Omega_x,$ such that $\mathcal{N}^{ (* )}_k = \{\Delta \vect{x}^{ (* )}_{k}, \Omega_{\theta}\}.$ Therefore by Lemma \ref{lem:lem_min}, $H (\Delta \vect{x}^{ (* )}_{k}, \vect{\theta}^{ (i+1 )}_{k} ) - H (\Delta \vect{x}^{ (* )}_{k}, \vect{\theta}^{ (i )}_{k} ) \leq 0$ for  $ (\Delta \vect{x}^{ (* )}_{k}, \vect{\theta}^{ (i+1 )}_{k} ),  (\Delta \vect{x}^{ (* )}_{k}, \vect{\theta}^{ (i )}_{k} ) \in \mathcal{N}^{ (* )}_k.$  Let $\vect{\theta}_k^{ (* )} \in \Omega_{\theta} $ be the converging point according to Lemma \ref{lem:lem_min}. Then, for  $ (\Delta \vect{x}_k^{ (* )}, \vect{\theta}^{ (* )}_k ),  (\Delta \vect{x}_k^{ (* )}, \vect{\theta}^{ (i )}_k ) \in \mathcal{M}^{ (* )}_k,$ $H (\Delta \vect{x}^{ (* )}_{k}, \vect{\theta}^{ (* )}_{k} ) - H (\Delta \vect{x}^{ (* )}_{k}, \vect{\theta}^{ (i )}_{k} ) \leq 0$ by Lemma~\ref{lem:lem_min} which provides the result.
\end{proof}

Next, we prove that the union of the two neighborhoods for each k $\mathcal{M}^{ (* )}_k \cup \mathcal{N}^{ (* )}_k,$ is non-empty.
\begin{lemma}
For any $k \in [0, K )$, let $\vect{\theta}^{ (* )}_{k} \in \Omega_\theta,$  be the minimizer of $H$ according to Lemma \ref{lem:lem_min} and define $\mathcal{M}^{ (* )}_k = \{\Omega_{x}, \vect{\theta}^{ (* )}_{k}\}.$ Similarly, let $\Delta \vect{x}_k^{ (* )} \in \Omega_x,$  be the maximizer of $H$ according to Lemma \ref{lem:lem_max} and define $\mathcal{N}^{ (* )}_k = \{\Delta \vect{x}^{ (* )}_{k}, \Omega_{\theta}\}.$ Then, $\mathcal{M}^{ (* )}_k \cup \mathcal{N}^{ (* )}_k$ is nonempty.
\label{lem:lem_nonE}
\end{lemma}

\begin{proof}
Let $\mathcal{M}^{ (* )}_k \cup \mathcal{N}^{ (* )}_k$  be empty. Then, for any $ (\Delta \vect{x}^{ (i+1 )}_k, \vect{\theta}^{ (. )}_{k} ),  (\Delta \vect{x}^{ (i )}_k,\vect{\theta}^{ (. )}_{k} ) \in \mathcal{M}^{ (* )}_k \cup \mathcal{N}^{ (* )}_k$, $H (\Delta \vect{x}^{ (i+1 )}_k, \vect{\theta}^{ (. )}_{k} ) - H (\Delta \vect{x}^{ (i )}_k,\vect{\theta}^{ (. )}_{k} )$ is undefined  because the union is empty. This contradicts Lemma \ref{lem:lem_min_opt}. Similarly, $H (\Delta \vect{x}^{ (. )}_k, \vect{\theta}^{ (i+1 )}_{k} ) - H (\Delta \vect{x}^{ (. )}_k,\vect{\theta}^{ (i )}_{k} )$  for $ (\Delta \vect{x}^{ (. )}_k, \vect{\theta}^{ (i+1 )}_{k} ),  (\Delta \vect{x}^{ (. )}_k,\vect{\theta}^{ (i )}_{k} ) \in \mathcal{M}^{ (* )}_k \cup \mathcal{N}^{ (* )}_k$ also contradicts Lemma \ref{lem:lem_max_opt}. Therefore, by contradiction,  $\mathcal{M}_k \cup \mathcal{N}_k$ cannot be empty.
\end{proof}

\subsection{Final Results}
We are now ready to present the main results. We show that there exists an equilibrium point~ (Theorem~1 ) and that the equilibrium point is stable~ (Theorem~2 ). 
\begin{thm}[Existence of an Equilibrium Point]
For any $k \in [0, K )$, let $\vect{\theta}^{ (* )}_{k} \in \Omega_\theta,$  be the minimizer of $H$ according to Lemma \ref{lem:lem_min_opt} and define $\mathcal{M}^{ (* )}_k = \{\Omega_{x}, \vect{\theta}^{ (* )}_{k}\}.$ Similarly, let $\Delta \vect{x}_k^{ (* )} \in \Omega_x,$  be the maximizer of $H$ according to Lemma \ref{lem:lem_max_opt} and define $\mathcal{N}^{ (* )}_k = \{\Delta \vect{x}^{ (* )}_{k}, \Omega_{\theta}\}.$ Further, let $\mathcal{M}^{ (* )}_k \cup \mathcal{N}^{ (* )}_k$ be nonempty according to Lemma.~\ref{lem:lem_nonE},  then $ (\Delta \vect{x}^{ (* )}_{k},\vect{\theta}^{ (* )}_{k} ) \in \mathcal{M}^{ (* )}_k \cup \mathcal{N}^{ (* )}_k$ is a local equilibrium point.
\end{thm}
\begin{proof}
By Lemma \ref{lem:lem_min_opt} we have at $ (\Delta \vect{x}_k^{ (* )}, \vect{\theta}^{ (* )}_k ),  (\Delta \vect{x}_k^{ (* )}, \vect{\theta}^{ (i )}_k ) \in \mathcal{M}^{ (* )}_k \cup \mathcal{N}^{ (* )}_k$ that
\begin{equation}
    \begin{aligned}
    H (\Delta \vect{x}_k^{ (* )}, \vect{\theta}^{ (* )}_k ) &\leq& H (\Delta \vect{x}_k^{ (* )}, \vect{\theta}^{ (i )}_k ).
	\end{aligned}
\end{equation}
Similarly, according to Lemma \ref{lem:lem_max_opt}, at $ (\Delta \vect{x}^{ (* )}_k, \vect{\theta}^{ (* )}_{k} ),  (\Delta \vect{x}^{ (i )}_k, \vect{\theta}^{ (* )}_{k} ) \in \mathcal{M}^{ (* )}_k \cup \mathcal{N}^{ (* )}_k$ we have
\begin{equation}
    \begin{aligned}
    H (\Delta \vect{x}^{ (* )}_k, \vect{\theta}^{ (* )}_{k} ) \geq H (\Delta \vect{x}^{ (i )}_k, \vect{\theta}^{ (* )}_{k} ).
	\end{aligned}
\end{equation}
Putting these inequalities together, we get
\begin{equation}
    \begin{aligned}
     H (\Delta \vect{x}_k^{ (* )}, \vect{\theta}^{ (i )}_k ) \geq H (\Delta \vect{x}_k^{ (* )}, \vect{\theta}^{ (* )}_k ) \geq H (\Delta \vect{x}_k^{ (i )}, \vect{\theta}^{ (* )}_k ),
	\end{aligned}
\end{equation}
which is the saddle point condition, and therefore $ (\Delta \vect{x}_k^{   (*  )}, \vect{\theta}^{(*)}_k )$ is a local equilibrium point in $\mathcal{M}^{ (* )}_k \cup \mathcal{N}^{ (* )}_k.$
\end{proof}
According to the preceeding theorem, there is at least one equillibrium point for the game summarized by $H$.

\begin{thm} [Stability of the Equilibrium Point]
For any $k \in [0,K)$, $\Delta \vect{x}^{ (i )}_{k} \in \Omega_x$ and $\vect{\theta}^{ (i )}_{k} \in \Omega_\theta$ be the initial values for $\Delta \vect{x}^{ (i )}_{k}$ and $\vect{\theta}^{ (i )}_{k}$ respectively. Define $\mathcal{M}_k = \{\Omega_{x}, \Omega_{\theta}\}$  with $H (\Delta \vect{x}^{ (i )}_{k}, \vect{\theta}^{ (i )}_{k} )$ given by Proposition \ref{prop2}. Let $\Delta \vect{x}^{ (i+1 )}_{k} - \Delta \vect{x}^{ (i )}_{k} = \alpha_{k}^{ (i )}\times  (\nabla_{\Delta \vect{x}^{ (i )}_{k}} H (\Delta \vect{x}^{ (i )}_{k}, \vect{\theta}^{ (. )}_{k} )  )/\| \nabla_{\Delta \vect{x}^{ (i )}_{k}} H (\Delta \vect{x}^{ (i )}_{k}, \vect{\theta}^{ (. )}_{k} ) \|^2 )$ and  $\vect{\theta}^{ (i+1 )}_{k} - \vect{\theta}^{ (i )}_{k} = -\alpha_{k}^{ (i )}\times \nabla_{{\vect{\theta}}^{ (i )}_{k}} H (\Delta \vect{x}^{ (. )}_{k}, \vect{\theta}^{ (i )}_{k} ).$ Let the existence of an equilibrium point be given by Theorem~1, then as a consequence of Lemma \ref{lem:lem_max} and \ref{lem:lem_min} $ (\Delta \vect{x}^{ (* )}_{k},\vect{\theta}^{ (* )}_{k} ) \in \mathcal{M}_k$ is a stable equilibrium point for  $H$.
 \label{thm:thm_st}
\end{thm}
\begin{proof}
Consider now the order of plays by the two players. By Lemma~\ref{lem:lem_max}, a game starting at $ (\Delta \vect{x}_k^{ (i )}, \vect{\theta}^{ (i )}_k ) \in \mathcal{M}_k$ will reach $ (\Delta \vect{x}_k^{ (* )}, \vect{\theta}^{ (i )}_k )$ which is a maximizer for $H.$ Now, define $\mathcal{N}_k =  (\Delta \vect{x}_k^{ (* )}, \Omega_{\theta}  ) \subset \mathcal{M}_k$ then a game starting at $ (\Delta \vect{x}_k^{ (* )}, \vect{\theta}^{ (i )}_k ) \in \mathcal{N}_k$ will converge to  $ (\Delta \vect{x}_k^{ (* )}, \vect{\theta}^{ (* )}_k ) \in \mathcal{N}_k$ according to Lemma~\ref{lem:lem_min}. Since,  $\mathcal{N}_k \subset \mathcal{M}_k,$ our result follows.
\end{proof}

\section{Experimental Details}
Much of this information is a repetition of details provided in \cite{Hsu18_EvalCL,vandeven2019generative}.

\begin{enumerate}
    \item \textit{Incremental Domain Learning~ (IDL ):} Incremental domain refers to the scenario when each new task introduces changes in the marginal distribution of the inputs. This scenario has been extensively discussed in the domain adaptation literature, where this shift in domain is typically referred to as ``non-stationary data  distribution" or domain shift. Overall, we aim to transfer knowledge from the old task to a new task where each task can be different in the sense of their marginal distribution.
    
    \item \textit{Incremental Class Learning~ (ICL ):} In this scenario, each task contains an exclusive subset of classes. The number of output nodes in a model equals the number of total classes in the task sequence. For instance, tasks could be constructed by using exactly one class from the MNIST data set where we aim to transfer knowledge from one class to another.
    
    \item \textit{Incremental Task Learning~ (ITL ):}  In this setup, the output spaces are disjoint between tasks[ for example, the previous task can be a classification problem of five classes, while the new task can be a regression. This scenario is the most generic and allows for the tasks to be defined arbitrarily. For each tasks, a model requires task-specific identifier$t$.
\end{enumerate}

\textbf{Split-MNIST}
For split-MNIST, the original MNIST-data set is split into five partitions where each partition is a two-way classification.  We pick 60000 images for training  (6000 per digit ) and 10000 images for test, i.e.  (1000 per digit ). For the incremental task learning in the split-MNIST experiment, the ten digits are split into five two-class classification tasks  (the model has five output heads, one for each task ) and the task identity  (1 to 5 ) is given for test. For the incremental class learning setup, we require the model to make a prediction over all classes  (digits 0 to 9 ). For the incremental domain learning, the model always predicts over two classes.

\textbf{Permuted-MNIST}
For permuted-MNIST, we permute the pixels in the MNIST data to create tasks where each task is a ten-way classification. The three CL scenarios that are generated for the permuted-MNIST are similar to the Split-MNIST data set except for the idea that the different tasks are now generated by applying random pixel permutations to the images. For incremental task learning, we use a multi-output strategy, and each task is attached to a task identifier. For incremental domain and class, we use a single output strategy and  each task as one where one of the 10 digits are predicted.  In incremental class learning, for each new task 10 new classes are generated by permuting the MNIST data set. For incremental task and domain, we use a total of 10 tasks whereas for incremental classes, we generate a total of 100 tasks.

\textbf{Network Architecture}
We keep our architecture identical to what is provided in   \cite{Hsu18_EvalCL,vandeven2019generative}. The loss function is categorical cross-entropy for classification.  All models were trained for $2$ epochs per task with a minibatch size of $128$ using the Adam optimizer  ($\beta_1 = 0.9$, $\beta_2 = 0.999$, learning rate$= 0.001$ ) as the default. For BCL, the size of the buffer  (i.e., a new task array ) $\cD_N (k )$ and a task memory array  (samples from all the previous tasks ) $\cD_{P} (k ) )$ is kept equivalent to naive rehearsal and other memory-driven approaches such as GEM and MER~ ($16,000$ samples ).

\textbf{Comparison Methods -- Baseline Strategies}
Additional details can be found from  \cite{Hsu18_EvalCL,vandeven2019generative}
\begin{enumerate}
    \item  A sequentially-trained neural network with different optimizers such as SGD, Adam~\cite{kingma2014adam}, and Adagrad~\cite{duchi2011adaptive}. 
    \item A standard $L_2-$regularized neural network where each task is shown sequentially.
    \item Naive rehearsal strategy  (experience replay ) where a replay buffer is initialized and data corresponding to all the previous tasks are stored. The buffer size is chosen to match the space overhead of online EWC and SI. 
\end{enumerate}
\textbf{Comparison Method-CL}
We compared the following CL methods:
 \begin{enumerate}
    \item \textbf{EWC~\cite{kirkpatrick2017overcoming} / Online EWC~\cite{schwarz2018progress} / SI~\cite{zenke2017continual}}: For these methods, a regularization term is added to the loss, with a hyperparameter used to control the regularization strength such that: $L (total ) = L (current ) + \lambda L (regularization ).$ $\lambda$ is set through a hyperparameter.
    
    \item \textbf{LwF~\cite{LwF} / DGR~\cite{DGR}}  Here, we set the loss to be $L (total ) = \alpha L (current ) +  (1 − \alpha )L (replay )$ where hyperparameter $\alpha$ is chosen according to how many tasks have been seen by the model.
    \item For \textbf{RtF}~\cite{vandeven2019generative}, MAS~\cite{aljundi2018memory}, GEM\cite{lopez2017gradient} and MER\cite{riemer2018learning}, we refer to the respective publication for details.
 \end{enumerate}
Additional details about the experiments can be found in \cite{Hsu18_EvalCL} as our paper retains their hyper-parameters and the experimental settings.

\bibliographystyle{plain}
\bibliography{cdc.bib}

\begin{thebibliography}{10}

\bibitem{abolfathi2021coachnet}
Elmira~Amirloo Abolfathi, Jun Luo, Peyman Yadmellat, and Kasra Rezaee.
\newblock {CoachNet}: An adversarial sampling approach for reinforcement
  learning, 2021.

\bibitem{aljundi2018memory}
Rahaf Aljundi, Francesca Babiloni, Mohamed Elhoseiny, Marcus Rohrbach, and
  Tinne Tuytelaars.
\newblock Memory aware synapses: Learning what (not) to forget.
\newblock In {\em Proceedings of the European Conference on Computer Vision
  (ECCV)}, pages 139--154, 2018.

\bibitem{beaulieu2020learning}
Shawn Beaulieu, Lapo Frati, Thomas Miconi, Joel Lehman, Kenneth~O Stanley, Jeff
  Clune, and Nick Cheney.
\newblock Learning to continually learn.
\newblock {\em arXiv preprint arXiv:2002.09571}, 2020.

\bibitem{bellman2015adaptive}
Richard~E Bellman.
\newblock {\em Adaptive control processes: a guided tour}.
\newblock Princeton university press, 2015.

\bibitem{caccia2021online}
Massimo Caccia, Pau Rodriguez, Oleksiy Ostapenko, Fabrice Normandin, Min Lin,
  Lucas Caccia, Issam Laradji, Irina Rish, Alexandre Lacoste, David Vazquez,
  and Laurent Charlin.
\newblock Online fast adaptation and knebowledge accumulation: a new approach
  to continual learning, 2021.

\bibitem{Caccia2020OnlineFA}
Massimo Caccia, Pau Rodr{\'i}guez, Oleksiy Ostapenko, Fabrice Normandin, Min
  Lin, Lucas Caccia, Issam~H. Laradji, Irina Rish, Alexande Lacoste, David
  V{\'a}zquez, and Laurent Charlin.
\newblock Online fast adaptation and knowledge accumulation: {A} new approach
  to continual learning.
\newblock {\em ArXiv}, abs/2003.05856, 2020.

\bibitem{carpenter1987massively}
Gail~A Carpenter and Stephen Grossberg.
\newblock A massively parallel architecture for a self-organizing neural
  pattern recognition machine.
\newblock {\em Computer vision, graphics, and image processing}, 37(1):54--115,
  1987.

\bibitem{chaudhry2020continual}
Arslan Chaudhry, Naeemullah Khan, Puneet~K Dokania, and Philip~HS Torr.
\newblock Continual learning in low-rank orthogonal subspaces.
\newblock {\em arXiv preprint arXiv:2010.11635}, 2020.

\bibitem{chaudhry2019continual}
Arslan Chaudhry, Marcus Rohrbach, Mohamed Elhoseiny, Thalaiyasingam Ajanthan,
  Puneet~K Dokania, Philip~HS Torr, and Marc'Aurelio Ranzato.
\newblock Continual learning with tiny episodic memories.
\newblock {\em arXiv preprint arXiv:1902.10486}, 2019.

\bibitem{chaudhry2019tiny}
Arslan Chaudhry, Marcus Rohrbach, Mohamed Elhoseiny, Thalaiyasingam Ajanthan,
  Puneet~K Dokania, Philip~HS Torr, and Marc'Aurelio Ranzato.
\newblock On tiny episodic memories in continual learning.
\newblock {\em arXiv preprint arXiv:1902.10486}, 2019.

\bibitem{doan2021theoretical}
Thang Doan, Mehdi~Abbana Bennani, Bogdan Mazoure, Guillaume Rabusseau, and
  Pierre Alquier.
\newblock A theoretical analysis of catastrophic forgetting through the {NTK}
  overlap matrix.
\newblock In {\em International Conference on Artificial Intelligence and
  Statistics}, pages 1072--1080. PMLR, 2021.

\bibitem{duchi2011adaptive}
John Duchi, Elad Hazan, and Yoram Singer.
\newblock Adaptive subgradient methods for online learning and stochastic
  optimization.
\newblock {\em Journal of Machine Learning Research}, 12(7), 2011.

\bibitem{ebrahimi2020adversarial}
Sayna Ebrahimi, Franziska Meier, Roberto Calandra, Trevor Darrell, and Marcus
  Rohrbach.
\newblock Adversarial continual learning, 2020.

\bibitem{farajtabar2019orthogonal}
Mehrdad Farajtabar, Navid Azizan, Alex Mott, and Ang Li.
\newblock Orthogonal gradient descent for continual learning, 2019.

\bibitem{fini2020online}
Enrico Fini, St{\'e}phane Lathuili{\`e}re, Enver Sangineto, Moin Nabi, and
  Elisa Ricci.
\newblock Online continual learning under extreme memory constraints.
\newblock In {\em European Conference on Computer Vision}, pages 720--735.
  Springer, 2020.

\bibitem{finn2017model}
Chelsea Finn, Pieter Abbeel, and Sergey Levine.
\newblock Model-agnostic meta-learning for fast adaptation of deep networks.
\newblock In {\em Proceedings of the 34th International Conference on Machine
  Learning-Volume 70}, pages 1126--1135. JMLR. org, 2017.

\bibitem{finn2019online}
Chelsea Finn, Aravind Rajeswaran, Sham Kakade, and Sergey Levine.
\newblock Online meta-learning.
\newblock {\em arXiv preprint arXiv:1902.08438}, 2019.

\bibitem{gupta2020lamaml}
Gunshi Gupta, Karmesh Yadav, and Liam Paull.
\newblock {La-MAML}: Look-ahead meta learning for continual learning, 2020.

\bibitem{Hsu18_EvalCL}
Yen-Chang Hsu, Yen-Cheng Liu, Anita Ramasamy, and Zsolt Kira.
\newblock Re-evaluating continual learning scenarios: A categorization and case
  for strong baselines.
\newblock In {\em NeurIPS Continual learning Workshop}, 2018.

\bibitem{javed2019meta}
Khurram Javed and Martha White.
\newblock Meta-learning representations for continual learning.
\newblock In {\em Advances in Neural Information Processing Systems}, pages
  1818--1828, 2019.

\bibitem{joseph2020metaconsolidation}
K~J Joseph and Vineeth~N Balasubramanian.
\newblock Meta-consolidation for continual learning, 2020.

\bibitem{jung2020continual}
Sangwon Jung, Hongjoon Ahn, Sungmin Cha, and Taesup Moon.
\newblock Continual learning with node-importance based adaptive group sparse
  regularization.
\newblock {\em arXiv e-prints}, pages arXiv--2003, 2020.

\bibitem{ke2020continual}
Zixuan Ke, Bing Liu, and Xingchang Huang.
\newblock Continual learning of a mixed sequence of similar and dissimilar
  tasks.
\newblock {\em Advances in Neural Information Processing Systems}, 33, 2020.

\bibitem{kingma2014adam}
Diederik~P Kingma and Jimmy Ba.
\newblock Adam: A method for stochastic optimization.
\newblock {\em arXiv preprint arXiv:1412.6980}, 2014.

\bibitem{kirkpatrick2017overcoming}
James Kirkpatrick, Razvan Pascanu, Neil Rabinowitz, Joel Veness, Guillaume
  Desjardins, Andrei~A Rusu, Kieran Milan, John Quan, Tiago Ramalho, Agnieszka
  Grabska-Barwinska, et~al.
\newblock Overcoming catastrophic forgetting in neural networks.
\newblock {\em Proceedings of the National Academy of Sciences},
  114(13):3521--3526, 2017.

\bibitem{knoblauch2020optimal}
Jeremias Knoblauch, Hisham Husain, and Tom Diethe.
\newblock Optimal continual learning has perfect memory and is {NP}-hard.
\newblock In {\em International Conference on Machine Learning}, pages
  5327--5337. PMLR, 2020.

\bibitem{lewis2012optimal}
Frank~L Lewis, Draguna Vrabie, and Vassilis~L Syrmos.
\newblock {\em Optimal control}.
\newblock John Wiley \& Sons, 2012.

\bibitem{LwF}
Zhizhong Li and Derek Hoiem.
\newblock Learning without forgetting.
\newblock {\em IEEE Transactions on Pattern Analysis and Machine Intelligence},
  40(12):2935--2947, 2017.

\bibitem{lin1992self}
Long-Ji Lin.
\newblock Self-improving reactive agents based on reinforcement learning,
  planning and teaching.
\newblock {\em Machine Learning}, 8(3-4):293--321, 1992.

\bibitem{lopez2017gradient}
David Lopez-Paz and Marc'Aurelio Ranzato.
\newblock Gradient episodic memory for continual learning.
\newblock In {\em Advances in neural information processing systems}, pages
  6467--6476, 2017.

\bibitem{GEM}
David Lopez-Paz and Marc'Aurelio Ranzato.
\newblock Gradient episodic memory for continual learning.
\newblock {\em arXiv preprint arXiv:1706.08840}, 2017.

\bibitem{mirzadeh2020understanding}
Seyed~Iman Mirzadeh, Mehrdad Farajtabar, Razvan Pascanu, and Hassan
  Ghasemzadeh.
\newblock Understanding the role of training regimes in continual learning.
\newblock {\em arXiv preprint arXiv:2006.06958}, 2020.

\bibitem{nagabandi2019deep}
Anusha Nagabandi, Chelsea Finn, and Sergey Levine.
\newblock Deep online learning via meta-learning: Continual adaptation for
  model-based {RL}, 2019.

\bibitem{nichol2018firstorder}
Alex Nichol, Joshua Achiam, and John Schulman.
\newblock On first-order meta-learning algorithms, 2018.

\bibitem{pan2020continual}
Pingbo Pan, Siddharth Swaroop, Alexander Immer, Runa Eschenhagen, Richard~E
  Turner, and Mohammad~Emtiyaz Khan.
\newblock Continual deep learning by functional regularisation of memorable
  past.
\newblock {\em arXiv preprint arXiv:2004.14070}, 2020.

\bibitem{riemer2018learning}
Matthew Riemer, Ignacio Cases, Robert Ajemian, Miao Liu, Irina Rish, Yuhai Tu,
  and Gerald Tesauro.
\newblock Learning to learn without forgetting by maximizing transfer and
  minimizing interference.
\newblock {\em arXiv preprint arXiv:1810.11910}, 2018.

\bibitem{rusu2016progressive}
Andrei~A Rusu, Neil~C Rabinowitz, Guillaume Desjardins, Hubert Soyer, James
  Kirkpatrick, Koray Kavukcuoglu, Razvan Pascanu, and Raia Hadsell.
\newblock Progressive neural networks.
\newblock {\em arXiv preprint arXiv:1606.04671}, 2016.

\bibitem{schwarz2018progress}
Jonathan Schwarz, Wojciech Czarnecki, Jelena Luketina, Agnieszka
  Grabska-Barwinska, Yee~Whye Teh, Razvan Pascanu, and Raia Hadsell.
\newblock Progress \& compress: A scalable framework for continual learning.
\newblock In {\em International Conference on Machine Learning}, pages
  4528--4537. PMLR, 2018.

\bibitem{DGR}
Hanul Shin, Jung~Kwon Lee, Jaehong Kim, and Jiwon Kim.
\newblock Continual learning with deep generative replay.
\newblock {\em arXiv preprint arXiv:1705.08690}, 2017.

\bibitem{snell2017prototypical}
Jake Snell, Kevin Swersky, and Richard Zemel.
\newblock Prototypical networks for few-shot learning.
\newblock In {\em Advances in neural information processing systems}, pages
  4077--4087, 2017.

\bibitem{titsias2019functional}
Michalis~K Titsias, Jonathan Schwarz, Alexander G de~G Matthews, Razvan
  Pascanu, and Yee~Whye Teh.
\newblock Functional regularisation for continual learning with {G}aussian
  processes.
\newblock {\em arXiv preprint arXiv:1901.11356}, 2019.

\bibitem{vandeven2019generative}
Gido~M. van~de Ven and Andreas~S. Tolias.
\newblock Generative replay with feedback connections as a general strategy for
  continual learning, 2019.

\bibitem{vinyals2016matching}
Oriol Vinyals, Charles Blundell, Timothy Lillicrap, Daan Wierstra, et~al.
\newblock Matching networks for one shot learning.
\newblock {\em Advances in Neural Information Processing Systems},
  29:3630--3638, 2016.

\bibitem{DBLP:journals/corr/abs-1806-06928}
Risto Vuorio, Dong{-}Yeon Cho, Daejoong Kim, and Jiwon Kim.
\newblock Meta continual learning.
\newblock {\em CoRR}, abs/1806.06928, 2018.

\bibitem{yao2020don}
Huaxiu Yao, Longkai Huang, Ying Wei, Li~Tian, Junzhou Huang, and Zhenhui Li.
\newblock Don't overlook the support set: Towards improving generalization in
  meta-learning.
\newblock {\em arXiv preprint arXiv:2007.13040}, 2020.

\bibitem{yin2020optimization}
Dong Yin, Mehrdad Farajtabar, Ang Li, Nir Levine, and Alex Mott.
\newblock Optimization and generalization of regularization-based continual
  learning: a loss approximation viewpoint, 2020.

\bibitem{yoon2017lifelong}
Jaehong Yoon, Eunho Yang, Jeongtae Lee, and Sung~Ju Hwang.
\newblock Lifelong learning with dynamically expandable networks.
\newblock {\em arXiv preprint arXiv:1708.01547}, 2017.

\bibitem{zenke2017continual}
Friedemann Zenke, Ben Poole, and Surya Ganguli.
\newblock Continual learning through synaptic intelligence.
\newblock In {\em Proceedings of the 34th International Conference on Machine
  Learning-Volume 70}, pages 3987--3995. JMLR. org, 2017.

\end{thebibliography}

\footnotesize
\begin{center}
    \framebox{\parbox{4.5in}{
    The submitted manuscript has been created by UChicago Argonne, LLC, Operator of Argonne National Laboratory (``Argonne''). Argonne, a U.S. Department of Energy Office of Science laboratory, is operated under Contract No. DE-AC02-06CH11357. The U.S. Government retains for itself, and others acting on its behalf, a paid-up nonexclusive, irrevocable worldwide license in said article to reproduce, prepare derivative works, distribute copies to the public, and perform publicly and display publicly, by or on behalf of the Government. The Department of Energy will provide public access to these results of federally sponsored research in accordance with the DOE Public Access Plan. \url{http://energy.gov/downloads/doe-public-access-plan}}}
    \normalsize
\end{center}
\end{document}